\documentclass[preprint,12pt]{elsarticle}

\usepackage{bm}
\usepackage{amsmath,amsfonts,amssymb,amsthm,enumerate}
\usepackage{xcolor}
\usepackage[margin=1in]{geometry}
\usepackage{mathrsfs}
\usepackage{hyperref}
\newtheorem{theorem}{Theorem}
\newtheorem{proposition}[theorem]{Proposition}

\newtheorem{lemma}[theorem]{Lemma}
\newtheorem{definition}[theorem]{Definition}

\journal{Neural Networks}

\begin{document}

\begin{frontmatter}



\title{Hypothesis Spaces for Deep Learning}


\author[1]{Rui Wang}
\ead{rwang11@jlu.edu.cn}
\affiliation[1]{organization={School of Mathematics, Jilin University},
            city={Changchun},
            postcode={130012}, 
            country={P. R. China}}

\author[2]{Yuesheng Xu\corref{cor1}}
\cortext[cor1]{This author is also a Professor Emeritus of Mathematics, Syracuse University, Syracuse, NY 13244, USA. All correspondence should be sent to this author.}
\ead{y1xu@odu.edu}  

\affiliation[2]{organization={Department of Mathematics and Statistics, Old Dominion University},
            city={Norfolk},
            postcode={23529}, 
            state={VA},
            country={USA}}

\author[2]{Mingsong Yan}
\ead{myan007@odu.edu}

\begin{abstract}
This paper introduces a hypothesis space for deep learning based on deep neural networks (DNNs). By treating a DNN as a function of two variables—the input variable and the parameter variable—we consider the set of DNNs where the parameter variable belongs to a space of weight matrices and biases determined by a prescribed depth and layer widths. To construct a Banach space of functions of the input variable, we take the weak* closure of the linear span of this DNN set. We prove that the resulting Banach space is a reproducing kernel Banach space (RKBS) and explicitly construct its reproducing kernel. Furthermore, we investigate two learning models—regularized learning and the minimum norm interpolation (MNI) problem—within the RKBS framework by establishing representer theorems. These theorems reveal that the solutions to these learning problems can be expressed as a finite sum of kernel expansions based on training data.
\end{abstract}




\end{frontmatter}


\section{Introduction}
Deep learning has achieved remarkable success across various applications. Mathematically, this success is underpinned by deep neural networks (DNNs)—multi-layered architectures that efficiently approximate complex functions, capture hierarchical structures, leverage depth for enhanced expressiveness, and benefit from effective optimization techniques.
The approximation capabilities of DNNs have been extensively studied in recent works \cite{daubechies2022nonlinear, huang2022error,li2024two, mhaskar2016deep, shen2022optimal,Xu_Multigrade2023,xu2022convergence, xu2024convergence, zhou2020universality}.
As pointed out in \cite{cucker2002mathematical}, learning processes do not take place in a vacuum. Classical learning methods are formulated within a reproducing kernel Hilbert space (RKHS) \cite{aronszajn1950theory}, where learning solutions are represented as a finite sum of kernel expansions in terms of training data \cite{scholkopf2001generalized}, utilizing universal kernels \cite{micchelli2006universal}. RKHSs serve as natural hypothesis spaces for classical learning, providing a rigorous mathematical  foundation for theoretical analysis. A fundamental and pressing question arises: what are the appropriate hypothesis spaces for deep learning? While recent studies have explored hypothesis spaces for shallow neural networks (networks of a single hidden layer) \cite{bartolucci2023understanding, chung2023barron, parhi2021banach, shenouda2024variation}, an analogous space for deep learning remains elusive.
This study aims to address this critical theoretical gap.

The construction of an appropriate hypothesis space for deep learning follows a systematic approach. We treat a DNN as a function of two variables: the input variable and the parameter variable. The set of DNNs is then considered as a collection of functions of the input variable, with the parameter variable ranging over the space of weight matrices and biases determined by the depth and layer widths of the network. By completing the linear span of this function set in the weak* topology, we construct a Banach space of functions of the input variable. We further establish that this resulting Banach space is a reproducing kernel Banach space (RKBS), where point evaluation functionals are continuous, and we construct an asymmetric reproducing kernel for the space, which depends on both the input and parameter variables. This RKBS is then proposed as the hypothesis space for deep learning. Notably, when deep neural networks reduce to shallow networks (i.e., with only one hidden layer), our hypothesis space coincides with the space for shallow learning studied in \cite{bartolucci2023understanding}.

In most existing approaches to  deep learning, the hypothesis space is implicitly defined as compositions of function spaces used to model shallow neural networks. However, the resulting space is not linear, lacking well-defined algebraic and topological structures. In contrast, we rigorously establish an RKBS as a hypothesis space for deep learning. This space not only possesses the necessary algebraic structure but also has a well-defined topology, making it suitable for analyzing learning methods. Moreover, the RKBS framework enables the representation of learning solutions through its reproducing kernel, which is identified with the DNN weighted by an appropriate function.

Upon introducing the hypothesis space for deep learning, we investigate two learning models, the  MNI problem and regularized learning within the resulting RKBS. We establish representer theorems for these learning models by leveraging the theory of the RKBS developed in \cite{xu2022sparse, xu2019generalized,  zhang2009reproducing} and the representer theorems for learning in general RKBSs established in \cite{ChengWangXu2023, wang2021representer, wang2023sparse}. 
Similar to classical representer theorems in RKHS-based learning, 
our results reveal that although these deep learning models operate in an infinite-dimensional setting, their solutions lie in finite-dimensional manifolds. More specifically, they can be expressed as a finite sum of kernel expansions based on training data. The representer theorems established in this paper are both kernel-based and data-dependent. Even when deep neural networks reduce to a shallow network, the corresponding representer theorem remains noval to the best of our knowledge. 
The proposed hypothesis space and the derived representer theorems provide valuable insights into deep learning, establishing a rigorous mathematical foundation that bridges the gap between empirical success and theoretical understanding. This foundation not only deepens our comprehension of deep learning mechanisms but also facilitates further advancements in model design, optimization, and generalization theory.

Representer theorems for deep learning have garnered significant interest in recent years; see, for example, \cite{bartolucci2023understanding,parhi2021banach,shenouda2024variation,bartolucci2024neural, bohn2019representer, parhi2022what,unser2019representer}. In \cite{bohn2019representer}, Bohn et al.\ considered the interpolation and regression problems within compositions of RKHSs and derived a concatenated representer theorem. This theorem characterizes learning solutions as a finite sum of kernel expansions based on training data.  Bartolucci et al.\ introduced an integral RKBS in \cite{bartolucci2023understanding} that 
models single-hidden-layer neural networks with scalar-valued output and established a representer theorem for learning in this RKBS framework. In \cite{bartolucci2024neural}, they extended their approach to deep learning by defining a deep integral RKBS—a ``nonlinear space" constructed via compositions of functions from shallow integral RKBSs—and derived a representer theorem using vector-valued measures. This theorem yields solutions in the form of DNNs with infinite width.
Parhi et al.\ explored representer theorems for single-hidden-layer neural networks of finite width with scalar-valued outputs in \cite{parhi2021banach}. Their results demonstrate that a shallow neural network provides the optimal solution to the learning problem. This work was later extended to deep neural networks in \cite{parhi2022what}, where it was shown that deep networks serve as solutions to regularization problems in a nonlinear function space formed by compositions of functions from a Banach space.
Shenouda et al.\ \cite{shenouda2024variation} established a representer theorem for shallow vector-valued neural networks by leveraging the composition of vector-valued variation spaces. Meanwhile, in \cite{unser2019representer}, Unser proposed optimizing activation functions in deep neural networks through second-order total-variation regularization. The resulting representer theorem ensures that the optimal activation function for each node in the network is a linear spline with adaptive knots.
These studies highlight the diverse approaches to representer theorems in deep learning, spanning compositions of RKHSs, integral RKBSs, and variation spaces.

This paper is organized into six sections and includes two appendices. We describe in Section \ref{section: Learning with Deep Neural Networks} a deep learning model with DNNs. In Section \ref{section: VV-RKBS}, we lay the groundwork for formulating RKBSs as hypothesis spaces for deep learning by elucidating the concept of vector-valued RKBSs. Section \ref{section: hypothesis space} is dedicated to the development of the hypothesis space for deep learning, where we demonstrate that the completion of the linear span of the DNN set, associated with the learning model, in a weak* topology forms an RKBS, serving as the hypothesis space for deep learning. In Section \ref{section: representer theorems}, we explore learning models within this RKBS framework, establishing representer theorems for the solutions of two learning models: regularized learning and the MNI problem. Finally, in Section \ref{section: concluding remarks}, we conclude with remarks on the advantages of learning within the proposed hypothesis space and the contributions of this paper. Appendix A provides necessary notions of convex analysis, while Appendix B reviews the explicit, data-dependent representer theorem for the MNI problem in a general Banach space setting, as established in \cite{wang2023sparse}.

\section{Learning with Deep Neural Networks}\label{section: Learning with Deep Neural Networks}
We present in this section a commonly used learning model based on DNNs, widely studied in the machine learning community.  

We begin by recalling the notation for DNNs.
Let $s$ and $t$ be positive integers. A DNN is a vector-valued function from $\mathbb{R}^s$ to $\mathbb{R}^t$, 
constructed through compositions of functions, each defined by an activation function applied to an affine transformation. Specifically, given a univariate function $\sigma: \mathbb{R}\to\mathbb{R}$,
we define the corresponding vector-valued function as
\begin{equation*}\label{activationF}
\sigma({x}):=[\sigma(x_1),\dots,\sigma(x_s)]^\top, \ \ \mbox{for}\ \ {x}:=[x_1, x_2,\dots, x_s]^\top\in\mathbb{R}^s.
\end{equation*}
For each $n\in\mathbb{N}$, let $\mathbb{N}_n:=\{1,2,\ldots,n\}$. For $k$ vector-valued functions $f_j$, $j\in\mathbb{N}_k$, where the range of $f_j$ is contained in the domain of $f_{j+1}$, for $j\in\mathbb{N}_{k-1}$, we denote their consecutive composition by
\begin{equation*}\label{consecutive_composition}
    \bigodot_{j=1}^k f_j:=f_k\circ f_{k-1}\circ\cdots\circ f_2\circ f_1,
\end{equation*}
which is defined on the domain of $f_1$. 
Fixing a prescribed $D\in \mathbb{N}$, we set $m_0:=s$ and $m_D:=t$, and specify  positive integers $m_j$, for $j\in \mathbb{N}_{D-1}$.
Given weight matrices $\mathbf{W}_j\in\mathbb{R}^{m_j\times m_{j-1}}$ and bias vectors $\mathbf{b}_j\in\mathbb{R}^{m_j}$ for $j\in\mathbb{N}_D$, 
a DNN is defined as
\begin{equation*}\label{DNN}
\mathcal{N}^D({x}):=\left(\mathbf{W}_D\bigodot_{j=1}^{D-1} \sigma(\mathbf{W}_j \cdot+\mathbf{b}_j)+\mathbf{b}_D\right)({x}),\ \ {x}\in\mathbb{R}^s.
\end{equation*}
Here, $x$ is the input vector, and $\mathcal{N}^D$ consists of $D-1$ hidden layers followed by an output layer at depth $D$.

A DNN can also be expressed recursively. From the definition above, we have the base case
\begin{equation*}\label{Step1}
    \mathcal{N}^1({x}):=\mathbf{W}_1 {x}+\mathbf{b}_1, \ \ {x}\in \mathbb{R}^s
\end{equation*}
and for all $j\in \mathbb{N}_{D-1}$, the recursive formula 
\begin{equation*}\label{Recursion}
    \mathcal{N}^{j+1}({x}):=\mathbf{W}_{j+1}\sigma(\mathcal{N}^j({x}))+\mathbf{b}_{j+1}, \ \ {x}\in \mathbb{R}^s.
\end{equation*}
To emphasize the dependence of $\mathcal{N}^D$ on its parameters, we write $\mathcal{N}^D$ as $\mathcal{N}^D(\cdot,\{\mathbf{W}_j,\mathbf{b}_j\}_{j=1}^D)$. In this paper, the set $\{\mathbf{W}_j,\mathbf{b}_j\}_{j=1}^D$ associated with the neural network $\mathcal{N}^D$ is always ordered according to its natural sequence as defined in $\mathcal{N}^D$. Throughout this paper, we assume that the activation function $\sigma$ is continuous.

It is beneficial to view the DNN $\mathcal{N}^D$ defined above as a function of two variables: the input variable $x\in \mathbb{R}^s$ and the parameter variable $\theta:=\{\mathbf{W}_j,\mathbf{b}_j\}_{j=1}^D$.
Given positive integers $m_j$, $j\in \mathbb{N}_{D-1}$, we define the width set as
\begin{equation}\label{width-set}
    \mathbb{W}:=\{m_j: j\in\mathbb{N}_{D-1}\}
\end{equation} 
and 
define the set of DNNs with $D$ layers by
\begin{equation}\label{Set_A}
    \mathcal{A}_{\mathbb{W}}:=\left\{ \mathcal{N}^D(\cdot,\{\mathbf{W}_j,\mathbf{b}_j\}_{j=1}^D): \mathbf{W}_j\in \mathbb{R}^{m_j\times m_{j-1}},  \mathbf{b}_j\in\mathbb{R}^{m_j}, j\in\mathbb{N}_D\right\}.
\end{equation}
Clearly, the set $\mathcal{A}_{\mathbb{W}}$ depends not only on $\mathbb{W}$ but also on $D$. For simplicity, we will omit the explicit dependence on $D$ in our notation whenever there is no risk of ambiguity. For example, we use $\mathcal{N}$ to denote $\mathcal{N}^D$.
Moreover, each element of $\mathcal{A}_{\mathbb{W}}$ is a vector-valued function mapping from $\mathbb{R}^s$ to  $\mathbb{R}^t$. 
To further reinterpret $\mathcal{A}_{\mathbb{W}}$,
we define the parameter space $\Theta$ as
\begin{equation}\label{Def:Theta}
\Theta=\Theta_{\mathbb{W}}:= \bigotimes_{j\in\mathbb{N}_D}(\mathbb{R}^{m_j\times m_{j-1}}\otimes \mathbb{R}^{m_j}).
\end{equation}
Note that $\Theta$ is measurable. 
For ${x}\in \mathbb{R}^s$ and $\theta\in \Theta$, we define
\begin{equation}\label{Def:kernel}
    \mathcal{N}({x},\theta):=\mathcal{N}^D({x},\{\mathbf{W}_j,\mathbf{b}_j\}_{j=1}^D).
\end{equation}
Since $\mathcal{N}({x},\theta)\in \mathbb{R}^t$, we can now express $\mathcal{A}_{\mathbb{W}}$ as
\begin{equation*}\label{Set_A*}
    \mathcal{A}_{\mathbb{W}}=\left\{ \mathcal{N}(\cdot,\theta): \theta\in \Theta_{\mathbb{W}} \right\}.
\end{equation*}


We now describe the  learning model with DNNs.
Suppose that a training dataset 
\begin{equation}\label{Dataset}
\mathbb{D}_m:=  \{(x_j,y_j) \in\mathbb{R}^s\times\mathbb{R}^t:j\in\mathbb{N}_m\}
\end{equation}
is given and we would like to train a neural network from the dataset. 
We denote by $\mathcal{L}(f,\mathbb{D}_m)$, for all functions $f$ from $\mathbb{R}^s$ to $\mathbb{R}^t$, a loss function determined by the dataset $\mathbb{D}_m$. For example, a loss function may take the form
\begin{equation*}\label{loss function}
    \mathcal{L}(f,\mathbb{D}_m):=\sum_{j\in\mathbb{N}_m}\|f(x_j)-y_j\|,
\end{equation*}
where $\|\cdot\|$ is a norm of $\mathbb{R}^t$. Given a loss function, a typical deep learning problem is to train the parameters $\theta\in \Theta_\mathbb{W}$ from the training dataset $\mathbb{D}_m$ by solving the optimization problem 
\begin{equation}\label{BasicLearningMethod}
    \min\{\mathcal{L}(\mathcal{N}(\cdot,\theta),\mathbb{D}_m): \theta\in\Theta_\mathbb{W}\},
\end{equation}
where $\mathcal{N}$ has the form in equation \eqref{Def:kernel}. Equivalently, optimization problem \eqref{BasicLearningMethod} may be written as
\begin{equation}\label{BasicLearningMethod-equivalent}
    \min\{\mathcal{L}(f,\mathbb{D}_m): f\in\mathcal{A}_{\mathbb{W}}\}.
\end{equation}

The optimization problem \eqref{BasicLearningMethod-equivalent} is a widely studied learning model in the machine learning community.
However, the set $\mathcal{A}_{\mathbb{W}}$ lacks an algebraic structure that guarantees closure under linear combinations and does not possess a topological structure that provides a norm. These limitations make the mathematical analysis of learning model \eqref{BasicLearningMethod-equivalent} particularly challenging.
As recently noted in \cite{Jentzen2022existence}, ``there is almost no result in the scientific
literature which actually establishes the existence of global minimizers of risk
functions'' in DNN training. 
The few existing results are limited to shallow networks with specific activation functions. For instance, the existence of a global minimizer was established for shallow networks with Heaviside activation in \cite{Kainen2000Best}. Similarly, under certain strong assumptions, \cite{Jentzen2022existence} proved the existence of a global minimizer for the loss function in the training of shallow networks with ReLU activation.
In general, the existence of a solution to problem \eqref{BasicLearningMethod-equivalent} might be analyzed through the following observations: If the set $\mathcal{A}_{\mathbb{W}}$ has sufficient capacity to contain a function $f^*\in\mathcal{A}_{\mathbb{W}}$ that perfectly interpolates the data,  then $f^*$ serves as a solution to problem \eqref{BasicLearningMethod-equivalent}, as it satisfies $\mathcal{L}(f^*,\mathbb{D}_m)=0.$ Additionally, if the continuous loss function $\mathcal{L}(\mathcal{N}(\cdot,\theta),\mathbb{D}_m)$ tends to infinity as the norm $\|\theta\|$ approaches infinity, then the existence of a solution to problem \eqref{BasicLearningMethod-equivalent} is guaranteed. Yet, these conditions are quite strong and often challenging to verify in practice.

We introduce a vector space that contains  $\mathcal{A}_{\mathbb{W}}$ and consider learning in the vector space. 
For this purpose, given a set $\mathbb{W}$ of layer widths defined by \eqref{width-set}, we define the set
\begin{equation}\label{space B Delta}
\mathcal{B}_{\mathbb{W}}:=\left\{ \sum_{l=1}^n c_l\mathcal{N}(\cdot,\theta_l): c_l\in\mathbb{R}, \theta_l\in \Theta_{\mathbb{W}}, l\in \mathbb{N}_n, n\in \mathbb{N}\right\}.
\end{equation}
In the next proposition, we present properties of $\mathcal{B}_{\mathbb{W}}$.

\begin{proposition}\label{prop: BM is a linear space}
If $\mathbb{W}$ is the width set defined by \eqref{width-set}, then 

(i) $\mathcal{B}_{\mathbb{W}}$ defined by \eqref{space B Delta} is the smallest vector space on $\mathbb{R}$ that contains the set $\mathcal{A}_{\mathbb{W}}$.

(ii) $\mathcal{B}_{\mathbb{W}}\subset\bigcup_{n\in\mathbb{N}}\mathcal{A}_{n\mathbb{W}}$. 
\end{proposition}
\begin{proof}
It is clear that $\mathcal{B}_{\mathbb{W}}$ may be identified as the linear span of $\mathcal{A}_{\mathbb{W}}$, that is, 
$$
\mathcal{B}_{\mathbb{W}}=\text{span} \left\{\mathcal{N}(\cdot,\theta): \theta\in \Theta_{\mathbb{W}}\right\}.
$$
Thus, $\mathcal{B}_{\mathbb{W}}$ is the smallest vector space containing $\mathcal{A}_{\mathbb{W}}$.

It remains to prove Item (ii).
To this end, we let $f\in\mathcal{B}_{\mathbb{W}}$. By the definition \eqref{space B Delta} of $\mathcal{B}_{\mathbb{W}}$, there exist $n'\in \mathbb{N}$, $c_l\in\mathbb{R}$, $\theta_l\in \Theta_\mathbb{W}$, for $l\in \mathbb{N}_{n'}$ such that 
$$
f(\cdot)=\sum_{l=1}^{n'} c_l\mathcal{N}(\cdot,\theta_l).
$$
It suffices to show that $f\in\mathcal{A}_{n'\mathbb{W}}$.
Noting that $\theta_l:=\{\mathbf{W}_j^l,\mathbf{b}_j^l\}_{j=1}^D$, for  $l\in\mathbb{N}_{n'}$,
we set 
\begin{equation*}
\widetilde{\mathbf{W}}_{1}:=\left[\begin{array}{c}
             \mathbf{W}_1^1  \\
              \mathbf{W}_1^2\\
              \vdots\\
              \mathbf{W}_1^{n'}
\end{array}\right],\ \
\widetilde{\mathbf{W}}_j:=\left[\begin{array}{cccc}
\mathbf{W}_j^1 & \mathbf{0} & \cdots & \mathbf{0} \\
\mathbf{0} & \mathbf{W}_j^2 & \cdots & \mathbf{0} \\
\vdots & \vdots & \ddots & \vdots \\
\mathbf{0} & \mathbf{0} & \cdots & \mathbf{W}_j^{n'}
\end{array}\right],\ j\in\mathbb{N}_{D-1}\backslash\{1\}, 
\end{equation*}
\begin{equation*}
\widetilde{\mathbf{W}}_D:=\left[\begin{array}{cccc}
c_1\mathbf{W}_D^1 & c_2\mathbf{W}_D^2 & \cdots & c_{n'}\mathbf{W}_D^{n'}
\end{array}\right],\ \
\widetilde{\mathbf{b}}_{j}:=\left[\begin{array}{c}
             \mathbf{b}_j^1  \\
              \mathbf{b}_j^2\\
              \vdots\\
              \mathbf{b}_j^{n'}
\end{array}\right],\ j\in\mathbb{N}_{D-1},
\ \ \widetilde{\mathbf{b}}_{D}:=\sum\limits_{j=1}^{n'} c_j \mathbf{b}_D^j. 
\end{equation*}
Clearly, we have that
$
\widetilde{\mathbf{W}}_{1}\in\mathbb{R}^{(n'm_1)\times {m_0}}$, 
$\widetilde{\mathbf{W}}_j\in\mathbb{R}^{(n'm_j)\times (n'm_{j-1})}$, $j\in\mathbb{N}_{D-1}\backslash\{1\}$, 
$\widetilde{\mathbf{W}}_D\in\mathbb{R}^{m_D\times (n'm_{D-1})}$, $\widetilde{\mathbf{b}}_{j}\in\mathbb{R}^{n'm_j}$, $j\in\mathbb{N}_{D-1}$, and $\widetilde{\mathbf{b}}_{D}\in\mathbb{R}^{m_D}$. 
Direct computation confirms that $f(\cdot)=\mathcal{N}(\cdot,\widetilde{\theta})$
with $\widetilde{\theta}:=\{\widetilde{\mathbf{W}}_j,\widetilde{\mathbf{b}}_j\}_{j=1}^D$. By definition \eqref{Set_A}, $f\in\mathcal{A}_{n'\mathbb{W}}$. 
\end{proof}

Proposition \ref{prop: BM is a linear space} reveals that $\mathcal{B}_{\mathbb{W}}$ is the smallest vector space that contains $\mathcal{A}_{\mathbb{W}}$. Hence, it serves as a reasonable substitute for  $\mathcal{A}_{\mathbb{W}}$.
Motivated by this result, we propose the following alternative learning model
\begin{equation}\label{GeneralLearningMethod}
    \inf\{\mathcal{L}(f,\mathbb{D}_m): f\in\mathcal{B}_{\mathbb{W}}\}.
\end{equation}
For a given width set $\mathbb{W}$, unlike the learning model \eqref{BasicLearningMethod-equivalent}, which searches for a minimizer in the set $\mathcal{A}_\mathbb{W}$, model \eqref{GeneralLearningMethod} seeks a minimizer in the vector space $\mathcal{B}_\mathbb{W}$. This space contains  $\mathcal{A}_\mathbb{W}$ and is itself contained in $\mathcal{A}:=\bigcup_{n\in\mathbb{N}}\mathcal{A}_{n\mathbb{W}}$. 
According to Proposition \ref{prop: BM is a linear space}, the learning model \eqref{GeneralLearningMethod} is ``semi-equivalent" to the learning model \eqref{BasicLearningMethod-equivalent} in the sense that 
\begin{equation}\label{Comparison-of-three-models}
    \mathcal{L}(\mathcal{N}_{\mathcal{A}},\mathbb{D}_m)\leq \mathcal{L}(\mathcal{N}_{\mathcal{B}_{\mathbb{W}}},\mathbb{D}_m)\leq \mathcal{L}(\mathcal{N}_{\mathcal{A}_{\mathbb{W}}},\mathbb{D}_m),
\end{equation}
where $\mathcal{N}_{\mathcal{B}_{\mathbb{W}}}$ is a minimizer of  model \eqref{GeneralLearningMethod}, $\mathcal{N}_{\mathcal{A}_{\mathbb{W}}}$ and $\mathcal{N}_{\mathcal{A}}$ are the minimizers of model \eqref{BasicLearningMethod-equivalent} and model \eqref{BasicLearningMethod-equivalent} with the set $\mathcal{A}_{\mathbb{W}}$ replaced by $\mathcal{A}$, respectively. One might argue that since model \eqref{BasicLearningMethod-equivalent} is a finite-dimensional optimization problem, while model \eqref{GeneralLearningMethod} is infinite-dimensional, the alternative model  \eqref{GeneralLearningMethod} may introduce unnecessary complexity. However, despite being infinite-dimensional, the algebraic structure of the vector space $\mathcal{B}_{\mathbb{W}}$ and the topological structure we will later equip it with offer significant advantages for mathematical analysis of learning on the space.
In fact, the vector-valued RKBS that we obtain by completing $\mathcal{B}_{\mathbb{W}}$ in a weak* topology will lead to a remarkable representer theorem,  which reduces the infinite-dimensional optimization problem to a finite-dimensional one. This provides a solution to the challenges posed by the infinite dimensionality of the space $\mathcal{B}_{\mathbb{W}}$.



\section{Vector-Valued Reproducing Kernel Banach Space}\label{section: VV-RKBS}
It was proved in the previous section that for a given width set $\mathbb{W}$, the set $\mathcal{B}_{\mathbb{W}}$, defined by \eqref{space B Delta}, is the smallest vector space containing  $\mathcal{A}_\mathbb{W}$. A key objective of this paper is to establish that the vector space $\mathcal{B}_{\mathbb{W}}$ is dense in a weak* topology within a vector-valued RKBS. To achieve this, we first introduce the notion of vector-valued RKBSs in this section. 

A Banach space $\mathcal{B}$ with norm $\|\cdot\|_{\mathcal{B}}$ is said to be a space of vector-valued functions on a given set $X$ if $\mathcal{B}$ consists of vector-valued functions defined on $X$ and satisfies the property that for every $f\in\mathcal{B}$, $\|f\|_{\mathcal{B}}=0$ implies $f({x}) = \mathbf{0}$ for all ${x}\in X$. 
For a Banach space $\mathcal{B}$ of vector-valued functions mapping from $X$ to $\mathbb{R}^n$, we define the point evaluation operator $\delta_{{x}}:\mathcal{B}\to\mathbb{R}^n$ for each ${x}\in X$  as 
\begin{equation*}
    \delta_{{x}}(f):=f({x}), \quad \mbox{for all}\ \ f\in\mathcal{B}. 
\end{equation*}
It is important to note that not all Banach spaces are function spaces. For instance,
when $X$ is a measure space with a positive measure $\mu$, the Banach spaces $L^p(X)$, $1\leq p\leq+\infty$, do not consist of functions but rather equivalence classes of functions with respect to
$\mu$.

We now formally define vector-valued RKBSs.

\begin{definition}\label{def: vector-valued RKBS} 
A Banach space $\mathcal{B}$ of vector-valued functions from $X$ to $\mathbb{R}^n$ is called a vector-valued RKBS if there exists a norm $\|\cdot\|$ on $\mathbb{R}^n$ such that for each $x\in X$, the point evaluation operator $\delta_x$ is continuous with respect to this norm on $\mathcal{B}$. That is, for each $x\in X$, there exists a constant $C_x>0$ 
such that 
\[
\|\delta_x(f)\|\leq C_x\|f\|_{\mathcal{B}}, \quad\text{ for all }\ f\in\mathcal{B}.
\]
\end{definition}

Since all norms on  $\mathbb{R}^n$ are equivalent, if a Banach space $\mathcal{B}$ of vector-valued functions from $X$ to $\mathbb{R}^n$ is a vector-valued RKBS with respect to one norm on $\mathbb{R}^n$, then it must also be a vector-valued RKBS with respect to any other norm  on  $\mathbb{R}^n$. Consequently, the continuity of point evaluation operators on $\mathcal{B}$ is independent of the specific choice of norm of the output space $\mathbb{R}^n$. 

The concept of RKBSs was originally introduced in \cite{zhang2009reproducing}, to ensure the stability of sampling process and to serve as a hypothesis space for sparse machine learning. Vector-valued RKBSs were further studied in \cite{lin2021multi, zhang2013vector}, where their definition involves an abstract Banach space, equipped with a specific norm, as the output space of functions. In Definition \ref{def: vector-valued RKBS}, we restrict the output space to the Euclidean space $\mathbb{R}^n$ without specifying a particular norm, leveraging the fundamental property that all norms on $\mathbb{R}^n$ are equivalent.

The following proposition establishes that point evaluation operators are continuous if and only if their component-wise point evaluation functionals are continuous. While this result can be derived from a general argument concerning $\mathbb{R}^n$-valued functions, we provide a complete proof for the convenience of the reader. To proceed, for a vector-valued function $f:X\to\mathbb{R}^n$,  we denote its $j$-th component by $f_j: X\to\mathbb{R}$, for each $j\in\mathbb{N}_n$, such that 
\begin{equation*}
    f({x}):=[f_j({x}):j\in\mathbb{N}_n]^\top,\quad {x}\in X. 
\end{equation*}
Moreover, for each $x\in X$, $k\in\mathbb{N}_n$, we define the linear functional $\delta_{{x},k}:\mathcal{B}\to\mathbb{R}$ by 
\begin{equation*}
    \delta_{{x},k}(f):=f_k({x}),\ \mbox{for}\  f:=[f_k:k\in\mathbb{N}_n]^{\top}\in\mathcal{B}. 
\end{equation*}

\begin{proposition}\label{prop: component wise for vector-valued RKBS}
    A Banach space $\mathcal{B}$ of vector-valued functions from $X$ to $\mathbb{R}^n$ is a vector-valued RKBS if and only if, for each ${x}\in X$, $k\in \mathbb{N}_n$, there exists a constant $C_{{x},k}>0$ such that 
    \begin{equation}\label{elementwise functionals}
        |\delta_{{x},k}(f)|\leq C_{{x},k}\|f\|_{\mathcal{B}}, \ \mbox{for all} \ f\in \mathcal{B}.
    \end{equation}
\end{proposition}
\begin{proof}
By Definition \ref{def: vector-valued RKBS} and the equivalence of norms in $\mathbb{R}^n$, the Banach space $\mathcal{B}$ is a vector-valued RKBS if and only if, for each ${x}\in X$, there exists a constant $C_x>0$ such that \begin{equation}\label{point-operator}
\|f(x)\|_{\infty}\leq C_{{x}}\|f\|_{\mathcal{B}},\ \mbox{for all} \  f\in\mathcal{B}.
\end{equation}
To establish the equivalence between inequalities \eqref{elementwise functionals} and \eqref{point-operator}, we proceed as follows. If \eqref{elementwise functionals} holds, then taking $C_x:=\max_{k\in\mathbb{N}_n}C_{x,k}$ yields \eqref{point-operator}. Conversely, if \eqref{point-operator} holds, then \eqref{elementwise functionals} follows with $C_{{x},k}:=C_x$.
\end{proof}

We now proceed to identify a reproducing kernel for a vector-valued RKBS, which requires introducing the notion of the $\delta$-dual space.  For a Banach space $B$ with norm $\|\cdot\|_{B}$, let $B^*$ denote its dual space, consisting of all continuous linear functionals on $B$, equipped with the norm 
$$
\|\nu\|_{B^*}:=\sup_{\|f\|_{B}\leq1}|\nu(f)|,\ \mbox{for all}\ \nu\in B^*.
$$
The dual bilinear form $\langle\cdot,\cdot\rangle_{B}$ on $B^*\times B$ is given by 
$$
\langle\nu,f\rangle_{ B}:=\nu(f),\ \mbox{for all}\ \nu\in B^* \  \mbox{and}\ f\in B.
$$
Now, let $\mathcal{B}$ be a vector-valued RKBS of functions from $X$ to $\mathbb{R}^n$, with dual space $\mathcal{B}^*$. Define 
\begin{equation}\label{Def:Delta}
    \Delta:=\mathrm{span}\{\delta_{x,j}:x\in X, j\in\mathbb{N}_n\}.
\end{equation}
By Proposition \ref{prop: component wise for vector-valued RKBS}, each $\delta_{x,j}$ is a continuous linear functional on $\mathcal{B}$, implying that $\Delta\subseteq\mathcal{B}^*$.
Let $\mathcal{B}'$ be the closure of $\Delta$ in the norm topology on $\mathcal{B}^*$, which we refer to as the $\delta$-dual space of $\mathcal{B}$. Clearly,  $\mathcal{B}'\subseteq\mathcal{B}^*$, and $\mathcal{B}'$ is the smallest Banach space containing all component-wise point evaluation functionals on $\mathcal{B}$. The $\delta$-dual space was originally introduced for scalar-valued RKBSs in \cite{xu2022sparse}.

For any Banach space $B$, one can construct a Banach space $B^{\sharp}$ of functions, defined on some domain $X'$, that is isometrically isomorphic to $B$. Specifically, each element $f\in B$ can be naturally interpreted as a
function on $X':=B^*$ via the mapping $f\rightarrow\phi_f$, where 
$$
\phi_f(\nu):=\nu(f),\  \mbox{for all}\  \nu\in B^*.
$$
Let $B^{\sharp}$ denote the space  of all such functions $\phi_f$ on $B^*$, endowed with the supremum norm 
$$
\|\phi_f\|_{B^{\sharp}}:=\sup_{\|\nu\|_{B^*}\leq1}|\phi_f(\nu)|,\ \mbox{for all}\ f\in B.
$$
Since
$$
\|\phi_f\|_{B^{\sharp}}=\sup_{\|\nu\|_{B^*}\leq1}|\nu(f)|=\|f\|_{B},
$$
it follows that $B^{\sharp}$, as a Banach space of functions on $B^*$, is isometrically isomorphic to $B$. As a direct consequence, the $\delta$-dual space $\mathcal{B}'$ is isometrically isomorphic to a Banach space of functions defined on domain $X':=(\mathcal{B}')^*$. Thus, there always exists a Banach space $\mathcal{B}^{\sharp}$ of functions that is isometrically isomorphic to the $\delta$-dual space $\mathcal{B}'$.  However, such a function space is not necessarily unique. 

The following proposition establishes that any given Banach space $\mathcal{B}^{\sharp}$ of functions, which is isometrically isomorphic to the $\delta$-dual space $\mathcal{B}'$, can be uniquely associated with a reproducing kernel $K$ for the RKBS $\mathcal{B}$. Let $\kappa: \mathcal{B}'\to \mathcal{B}^{\sharp}$ be the isometric isomorphism. We define a bilinear form $\langle \cdot,\cdot\rangle_{\mathcal{B}^{\sharp}\times\mathcal{B}}$ on $\mathcal{B}^{\sharp}\times\mathcal{B}$ as
\begin{equation*}\label{bilinear-form}
\langle g,f\rangle_{\mathcal{B}^{\sharp}\times\mathcal{B}}:=\langle \kappa^{-1}(g),f\rangle_{\mathcal{B}},\ \mbox{for all}\ g\in\mathcal{B}^{\sharp}, f\in\mathcal{B}.  
\end{equation*}

\begin{proposition}\label{existence of reproducing kernel}

Suppose that  $\mathcal{B}$ is a vector-valued RKBS of functions from $X$ to $\mathbb{R}^n$, and let $\mathcal{B}'$ be its $\delta$-dual space. For each Banach space $\mathcal{B}^{\sharp}$ of functions from $X'$ to $\mathbb{R}$ that is isometrically isomorphic to $\mathcal{B}'$, there exists a unique vector-valued function $K:X\times X'\to\mathbb{R}^n$ satisfying the following properties: 

(i) $K_j(x,\cdot)\in \mathcal{B}^{\sharp}$ for all $x\in X$, $j\in\mathbb{N}_n$;

(ii) the reproducing property holds
\begin{equation}\label{def: reproducing property}
    f_j(x)=\langle K_j(x,\cdot),f\rangle_{\mathcal{B}^{\sharp}\times\mathcal{B}},\ \mbox{for all}\ f=[f_j:j\in\mathbb{N}_n]\in\mathcal{B} \ \mbox{and all}\  x\in X, \ j\in\mathbb{N}_n.
\end{equation}
\end{proposition}
\begin{proof}
We construct a vector-valued function $K: X\times X'\to\mathbb{R}^n$ that satisfies the desired properties. Let $\kappa$ be the isometric isomorphism from $\mathcal{B}'$ to $\mathcal{B}^{\sharp}$. Since $\delta_{x,j}\in\mathcal{B}'$ for each $x\in X$ and $j\in\mathbb{N}_n$, it follows that $\kappa(\delta_{x,j})\in\mathcal{B}^{\sharp}$, and we obtain 
\begin{equation}\label{K_xj}
    f_j(x)=\langle\delta_{x,j},f\rangle_{\mathcal{B}}=\langle \kappa(\delta_{x,j}),f\rangle_{\mathcal{B}^{\sharp}\times\mathcal{B}}, \ \mbox{for all}\ f\in\mathcal{B}.
    \end{equation}
We define the $j$-th component of the vector-valued function $K: X\times X'\to\mathbb{R}^n$ by
\begin{equation}\label{def-K}
   K_j(x,x'):=\kappa(\delta_{x,j})(x'),\ x\in X, \ x'\in X',\ j\in\mathbb{N}_n. 
\end{equation}
By construction, we have that $K_j(x,\cdot)\in\mathcal{B}^{\sharp}$ for all $x\in X$ and $j\in\mathbb{N}_n$. Substituting \eqref{def-K} into the right-hand side of \eqref{K_xj} yields the reproducing property in \eqref{def: reproducing property}. 

The uniqueness of the vector-valued function $K$ follows directly from the properties of the isometric isomorphism $\kappa$. The details are left to the interested reader.
%
\end{proof}

We refer to the vector-valued function $K:X\times X'\rightarrow\mathbb{R}^n$, which satisfies $K_j(x,\cdot)\in\mathcal{B}^{\sharp}$ for all $x\in X$, $j\in\mathbb{N}_n$, and the reproducing property given by equation \eqref{def: reproducing property}, as the reproducing kernel for the vector-valued RKBS $\mathcal{B}$. Furthermore, equation \eqref{def: reproducing property} is known as the reproducing property. Clearly, it follows that $K(x,\cdot)\in(\mathcal{B}^{\sharp})^n$ for all $x\in X$. 

As established in Proposition \ref{existence of reproducing kernel}, each Banach space $\mathcal{B}^{\sharp}$ of functions that is isometrically isomorphic to $\mathcal{B}'$ corresponds to a unique reproducing kernel $K$ for the RKBS $\mathcal{B}$. 
The choice
of $\mathcal{B}^{\sharp}$ affects the specific form of 
$K$, meaning that different selections of $\mathcal{B}^{\sharp}$ yield different reproducing kernels. The pair $(\mathcal{B}^{\sharp},K)$ enables  the reproduction of the point evaluation operator as expressed in equation \eqref{def: reproducing property}. The reproducing kernel plays a fundamental role in both theoretical analysis and practical applications. While a vector-valued RKBS may admit multiple reproducing kernels, selecting an appropriate one is essential. Beyond merely reproducing point evaluation functionals or operators, the reproducing kernel $K(\cdot, x')$, for $x'\in X'$, should also effectively represent solutions to learning problems. The concept of the vector-valued RKBS and its reproducing kernel will provide a foundation for understanding the hypothesis space in deep learning, as discussed in the next section.

It is important to note that although $\mathcal{B}$ is a space of vector-valued functions, the $\delta$-dual space  $\mathcal{B}'$ defined here consists of scalar-valued functions.  This follows from the form of the point evaluation functionals in set $\Delta$ defined in \eqref{Def:Delta}. The definition of the $\delta$-dual space for the vector-valued RKBS $\mathcal{B}$ is not unique—one could alternatively define it as a space of vector-valued functions. In this paper, we adopt the current form of $\mathcal{B}'$ due to its simplicity and sufficiency for our purposes. Other possible formulations of the $\delta$-dual space will be explored in future work.

\section{Hypothesis Space}\label{section: hypothesis space}

In this section, we return to understanding the vector space $\mathcal{B}_\mathbb{W}$ introduced in Section \ref{section: Learning with Deep Neural Networks} from the RKBS viewpoint. Specifically, our goal is to introduce a vector-valued RKBS in which the vector space $\mathcal{B}_\mathbb{W}$ is weakly$^*$ dense. The resulting vector-valued RKBS will serve as the hypothesis space for deep learning.

We first construct the vector-valued RKBS. Recalling the parameter space $\Theta$ defined by equation \eqref{Def:Theta}, we use $C_0(\Theta)$ to denote the space of the continuous {\it scalar-valued} functions vanishing at infinity on $\Theta$. We equip the supremum norm on $C_0(\Theta)$, namely, $\|f\|_{\infty}:=\sup_{\theta\in\Theta}|f(\theta)|$, for all $f\in C_0(\Theta)$. For the function $\mathcal{N}(x,\theta)$, $x\in\mathbb{R}^s$, $\theta\in\Theta$, defined by equation \eqref{Def:kernel}, we denote by $\mathcal{N}_k({x},\theta)$ the $k$-th component of $\mathcal{N}({x},\theta)$, for $k\in\mathbb{N}_t$.  
We require that all components $\mathcal{N}_k({x},\cdot)$ with a weight belong to $C_0(\Theta)$ for all $x\in\mathbb{R}^s$. Specifically, we assume that there exists
a continuous weight function $\rho:\Theta\to\mathbb{R}^+:=(0,\infty)$ such that the functions
\begin{equation}\label{requirement}
\mathcal{N}_k({x},\cdot)\rho(\cdot)\in C_0(\Theta), \ \ \mbox{for all} \ \ x\in \mathbb{R}^s, \ k\in\mathbb{N}_t.
\end{equation}
An example of such a weight function is given by the rapidly decreasing function
\begin{equation}\label{gaussian weight function}
    \rho(\theta):=\exp(-\|\theta\|_2^2), \ \ \theta\in\Theta.
\end{equation}
Requirement \eqref{requirement} in fact imposes a hypothesis to the activation function $\sigma$: (i) $\sigma$ is continuous and (ii) when the weight function $\rho$ is chosen as \eqref{gaussian weight function}, we need to select the activation function $\sigma$ having a growth rate no greater than polynomials. We remark that many commonly used activation functions satisfy this requirement. They include the ReLU function 
$$
\sigma(x):=\max\{0,x\},\ \ x\in\mathbb{R},
$$ 
and the sigmoid function 
$$
\sigma(x):=\frac{1}{1+e^{-x}},\ \ x\in\mathbb{R}.
$$

We need a measure on the set $\Theta$.
A Radon measure \cite{folland1999real} on $\Theta$ is a Borel measure on $\Theta$ that is finite on all compact sets of $\Theta$, outer regular on all Borel sets of $\Theta$, and inner regular on all open sets of $\Theta$. Let $\mathcal{M}(\Theta)$ denote the space of finite Radon measures, equipped with the total variation norm 
\begin{equation}\label{TVnorm}
    \|\mu\|_{\mathrm{TV}}:=\sup\left\{\sum_{k=1}^\infty\left|\mu(E_k)\right|:\Theta=\bigcup_{k=1}^\infty E_k,\ E_i\cap E_j=\emptyset\text{ whenever }i\neq j \right\},\ \  \mu\in\mathcal{M}(\Theta),
\end{equation}
where $E_k$ are required to be measurable.
Note that  $\mathcal{M}(\Theta)$ is the dual space of  $C_0(\Theta)$ (see, for example, \cite{conway2019course}) and is therefore a Banach space. Moreover, the dual bilinear form on $\mathcal{M}(\Theta)\times C_0(\Theta)$ is given by 
\begin{equation}\label{DualBilinearForm}
    \langle \mu,g\rangle_{C_0(\Theta)}:=\int_{\Theta} g(\theta)d\mu(\theta),\ \text{for }\mu\in\mathcal{M}(\Theta),\ g\in C_0(\Theta). 
\end{equation}
%
For $\mu \in \mathcal{M}(\Theta)$, we let 
\begin{equation}\label{Def:f_mu^k}
    f_{\mu}^k(\cdot):=\int_\Theta \mathcal{N}_k(\cdot,\theta)\rho(\theta)d\mu(\theta), \ \  k\in \mathbb{N}_t, 
\end{equation}
and 
\begin{equation*}\label{vector-valued fmu}
f_\mu(\cdot):=\left[f_{\mu}^k(\cdot): k\in\mathbb{N}_t\right]^\top.
\end{equation*}
We introduce the vector space
\begin{equation}\label{banach space DNN}
    \mathcal{B}_{\mathcal{N}}:=\left\{f_\mu:\mu \in \mathcal{M}(\Theta)\right\},
\end{equation}
with norm 
\begin{equation}\label{banach space norm DNN}
    \|f_\mu\|_{\mathcal{B}_{\mathcal{N}}}:=\inf\left\{\|\nu\|_{\mathrm{TV}}:f_\nu=f_\mu,\ \nu\in\mathcal{M}(\Theta)\right\},
\end{equation}
where $f_\mu^k$, $k\in\mathbb{N}_t$, are defined by equation \eqref{Def:f_mu^k} and $\|\cdot\|_\mathrm{TV}$ is defined as \eqref{TVnorm}. Note that in definition \eqref{banach space norm DNN} of the norm $\|f_\mu\|_{\mathcal{B}_{\mathcal{N}}}$, the infimum is taken over all the measures $\nu\in\mathcal{M}(\Theta)$ that satisfy $t$ equality constraints
\begin{equation*}
    \int_\Theta \mathcal{N}_k(\cdot,\theta)\rho(\theta)d\mu(\theta)=\int_\Theta \mathcal{N}_k(\cdot,\theta)\rho(\theta)d\nu(\theta),\quad k\in\mathbb{N}_t. 
\end{equation*}
In particular, in the case $t=1$, where $f_\mu$ reduces to a neural network of a scalar-valued output, the norm $ \|f_\mu\|_{\mathcal{B}_{\mathcal{N}}}$ is taken over the measures $\nu\in\mathcal{M}(\Theta)$ that satisfies only a single equality constraint. 
The construction  of the Banach space of functions described above follows a standard recipe that can be traced back to the variation spaces \cite{barron1993universal,DeVore1998nonlinear,siegel2023Characterization}. Banach spaces so constructed were widely utilized to understand the deep neural networks \cite{bartolucci2023understanding,parhi2021banach,shenouda2024variation,bartolucci2024neural,parhi2022what,Bach2017breaking,spek2023duality}. Especially, the special case of $\mathcal{B}_{\mathcal{N}}$ with $\mathcal{N}$ being a neural network of a {\it single} hidden layer was recently studied in \cite{bartolucci2023understanding}. 

Elements of $\mathcal{B}_{\mathcal{N}}$ are vector-valued functions of the input variable. 
Since $\Theta=\Theta_\mathbb{W}$, these elements depend on the layer width set $\mathbb{W}$. Consider two different sets of layer widths, $\mathbb{W}$ and  $\widetilde{\mathbb{W}}$. The corresponding function spaces $\mathcal{B}_{\mathcal{N}}$ and $\mathcal{B}_{\widetilde{\mathcal{N}}}$, are generally distinct. In particular, if $\mathbb{W}\preccurlyeq\widetilde{\mathbb{W}}$, meaning that $m_j\leq \widetilde{m}_j$ for all $j\in\mathbb{N}_{D-1}$, then - using zero-padding for weight matrices and bias vectors - it can be shown, as in \cite{XuZhang2024}, that $\mathcal{N}(\cdot,\theta)=\widetilde{\mathcal{N}}(\cdot,\widetilde{\theta})$ for $\theta\in\Theta_\mathbb{W}$ and $\widetilde{\theta}:=[\theta,\mathbf{0}]\in\Theta_{\widetilde{\mathbb{W}}}$. In this case, if $\widetilde{\rho}([\theta,\mathbf{0}])=\rho(\theta)$ for all $\theta\in\Theta_{\mathbb{W}}$, it can be shown that $\mathcal{B}_{\mathcal{N}}\subseteq\mathcal{B}_{\widetilde{\mathcal{N}}}$. Further discussion on this topic will be provided after we establish that $\mathcal{B}_{\mathcal{N}}$
is an RKBS later in this section.

Next, we establish that the space $\mathcal{B}_{\mathcal{N}}$, defined by \eqref{banach space DNN} with norm \eqref{banach space norm DNN}, is a Banach space that possesses a pre-dual space. This is achived by demonstrating that   $\mathcal{B}_{\mathcal{N}}$ is isometrically isomorphic to a quotient space. We begin by recalling the concept of a quotient space. Let $B$ be a Banach space with its dual space $B^*$ and let $M$ be a closed subspace of $B$. For each $f\in B$, the set $f+M$, which contains $f$, is called the coset of $M$.  The quotient space $B/M$ is then defined as  $B/M:=\{f+M: f\in B\}$, with the quotient norm given by 
\begin{equation*}
    \|f+M\|_{B/M}:=\inf\left\{\|f+g\|:g\in M\right\},\quad f\in B.
\end{equation*}
It is well-known \cite{megginson2012introduction} that $B/M$ is itself a Banach space. A Banach space $B$ is said to have a {\it pre-dual space} if there exists a Banach space $B_*$ such that $(B_*)^*=B$, in which case $B_*$ is referred to as a pre-dual space of $B$. 
We also introduce the notion of annihilators. Given subsets $M \subset B$ and $M' \subset B^*$, the annihilator of $M$ in $B^*$ is defined as 
$$
M^\perp:=\{\nu\in B^*:\langle\nu,f\rangle_{B}=0, \ \text{for all }f\in M\}.
$$
Similarly, the annihilator of $M'$ in $B$ is given by
$$
^{\perp}M':=\{f\in B: \langle\nu,f\rangle_{B}=0,
\ \mbox{for all}\  \nu\in M'\}.
$$  
We now review a result concerning the dual space of a closed subspace of a Banach space. Specifically, let $M$ be a closed subspace of a Banach space $B$. For any $\nu\in B^*$, we denote by $\nu|_{M}$ the restriction of $\nu$ to $M$. It is clear that $\nu|_{M}\in M^*$ and satisfies $\|\nu|_{M}\|_{M^*}\leq\|\nu\|_{B^*}$. Moreover, the dual space $M^*$ can be identified as $B^*/M^\perp$. In fact, by Theorem 10.1 in Chapter III of \cite{conway2019course}, the mapping $\tau:B^*/M^\perp\to M^*$ defined by 
$$
\tau(\nu+M^\perp):=\nu|_{M}, \ \mbox{for}\ \nu\in B^*,
$$ 
is an isometric isomorphism between $B^*/M^\perp$ and $M^*$.

To establish that $\mathcal{B}_{\mathcal{N}}$ is a Banach space, we identify a quotient space that is isometrically isomorphic to $\mathcal{B}_{\mathcal{N}}$. To this end, we introduce the closed subspace $\mathcal{S}$ of  $C_0(\Theta)$, defined as 
\begin{equation}\label{subspace of C0}
\mathcal{S}:=\overline{\mathrm{span}}\{\mathcal{N}_k({x},\cdot)\rho(\cdot): {x}\in\mathbb{R}^s,k\in\mathbb{N}_t\},
\end{equation}
where the closure is taken in the supremum norm. By definition, it is evident that $\mathcal{S}$ is a Banach space consisting of functions defined on the parameter space $\Theta$ and $\mathcal{S}^\perp$ is a closed subspace of $\mathcal{M}(\Theta)$.

\begin{proposition}\label{prop: BN is isometic isomorphic to quotient space}
Let $\Theta$ be the parameter space defined by \eqref{Def:Theta}, and let $\rho:\Theta\to\mathbb{R}^+$ be a continuous function on $\Theta$. If, 
for each $x\in \mathbb{R}^s$ and $k\in\mathbb{N}_t$, the function $\mathcal{N}_k({x},\cdot)\rho(\cdot)\in C_0(\Theta)$, then the space $\mathcal{B}_{\mathcal{N}}$, defined by \eqref{banach space DNN} endowed with the norm \eqref{banach space norm DNN}, is a Banach space with a pre-dual space $\mathcal{S}$ defined by \eqref{subspace of C0}.
\end{proposition}
\begin{proof}
It is clear that $\mathcal{B}_{\mathcal{N}}$ is a normed space endowed with the norm $\|\cdot\|_{\mathcal{B}_{\mathcal{N}}}$. To prove completeness, we show that $\mathcal{B}_{\mathcal{N}}$ is isometrically isomorphic to the quotient space $\mathcal{M}(\Theta)/\mathcal{S}^\perp$, which is a Banach space.

We begin by characterizing the annihilator $\mathcal{S}^\perp$ in $\mathcal{M}(\Theta)$. By the definition \eqref{subspace of C0}, a measure $\mu \in \mathcal{S}^\perp$ satisfies
$\langle\mu,\mathcal{N}_k({x},\cdot)\rho(\cdot)\rangle_{C_0(\Theta)}=0$ for all $x\in \mathbb{R}^s$ and $k\in\mathbb{N}_t$. 
By the definition \eqref{Def:f_mu^k} of $f_{\mu}^k$ and the dual bilinear form \eqref{DualBilinearForm}, this is equivalent to $f_{\mu}^k(x) = 0$ for all $x\in \mathbb{R}^s$, which implies $f_{\mu} = 0$. 
Consequently, we obtain
\begin{equation}\label{Sperp}
    \mathcal{S}^\perp=\{\mu\in\mathcal{M}(\Theta): f_{\mu}=0\}.
\end{equation}
Now, define the mapping $\varphi$ from $\mathcal{B}_{\mathcal{N}}$ to $\mathcal{M}(\Theta)/\mathcal{S}^\perp$ by 
\begin{equation}\label{def: mapping phi}\varphi(f_\mu):=\mu+\mathcal{S}^\perp, \quad\mu\in\mathcal{M}(\Theta).
\end{equation}
We show that $\varphi$ is an isometric isomorphism. 


We first show that $\varphi$ is an isometry. For any $f_{\mu}\in\mathcal{B}_\mathcal{N}$ with $\mu\in\mathcal{M}(\Theta)$, another measure $\nu\in\mathcal{M}(\Theta)$ satisfies $f_{\nu}=f_{\mu}$ if and only if $f_{\nu-\mu}=0$, which by \eqref{Sperp} is equivalent to $\nu-\mu\in \mathcal{S}^\perp$, i.e., $\nu=\mu+\mu'$ for some $\mu'\in \mathcal{S}^\perp$. Hence, from \eqref{banach space norm DNN},
$$
\|f_{\mu}\|_{\mathcal{B}_{\mathcal{N}}}=\inf\left\{\|\mu+\mu'\|_{\mathrm{TV}}:\mu'\in \mathcal{S}^\perp\right\}.
$$
By the definition of the quotient norm, we obtain
$$
\|f_{\mu}\|_{\mathcal{B}_{\mathcal{N}}}=\|\mu+\mathcal{S}^\perp\|_{\mathcal{M}(\Theta)/\mathcal{S}^\perp}.
$$
Thus, from \eqref{def: mapping phi}, we conclude that $\|\varphi(f_{\mu})\|_{\mathcal{M}(\Theta)/\mathcal{S}^\perp}=\|f_{\mu}\|_{\mathcal{B}_{\mathcal{N}}}$. This shows that $\varphi$ is an isometry. 

We further establish the bijectivity of $\varphi$.
Since $\varphi$ is an isometry, it is injective. Clearly,  $\varphi$ is surjective. Hence, it is bijective.
Consequently, $\varphi$ is an isometric isomorphism between $\mathcal{B}_{\mathcal{N}}$ and the Banach space $\mathcal{M}(\Theta)/\mathcal{S}^\perp$,  proving that $\mathcal{B}_{\mathcal{N}}$ is complete.  

Next, we establish that $\mathcal{B}_{\mathcal{N}}$ is isometrically isomorphic to the dual space of $\mathcal{S}$. Since $\mathcal{S}$ is a closed subspace of $C_0(\Theta)$ and $(C_0(\Theta))^*=\mathcal{M}(\Theta)$, by Theorem 10.1 in \cite{conway2019course} with $B:=C_0(\Theta)$ and $M:=\mathcal{S}$, the mapping 
\begin{equation*}\label{pho}
   \tau:\mathcal{M}(\Theta)/\mathcal{S}^\perp\to \mathcal{S}^*, \ \      \tau(\mu+\mathcal{S}^\perp):=\mu|_{\mathcal{S}},\quad\mu\in\mathcal{M}(\Theta)
\end{equation*}
is an isometric isomorphism. 
Because $\varphi$ in 
\eqref{def: mapping phi} is an isometric isomorphism from $\mathcal{B}_{\mathcal{N}}$ to $\mathcal{M}(\Theta)/\mathcal{S}^\perp$, the composition $\tau\circ\varphi$ is an isometric isomorphism from  $\mathcal{B}_{\mathcal{N}}$ to $\mathcal{S}^*$. Thus, $\mathcal{S}$ is a pre-dual space of $\mathcal{B}_{\mathcal{N}}$.
\end{proof}

The space $\mathcal{B}_{\mathcal{N}}$ with the norm $\|\cdot\|_{\mathcal{B}_{\mathcal{N}}}$, guaranteed by Proposition \ref{prop: BN is isometic isomorphic to quotient space}, is a Banach space. We denote by $\mathcal{B}_{\mathcal{N}}^*$ the dual space of $\mathcal{B}_{\mathcal{N}}$ endowed with the norm  
\begin{equation}\label{norm of BN*}
    \|\ell\|_{\mathcal{B}_{\mathcal{N}}^*}=\sup\{|\langle\ell,f_\mu\rangle_{\mathcal{B}_{\mathcal{N}}}|:\|f_\mu\|_{\mathcal{B}_{\mathcal{N}}}=1, f_\mu\in\mathcal{B}_{\mathcal{N}}\}, \ \ \mbox{for}\ \ \ell\in \mathcal{B}_{\mathcal{N}}^*.
\end{equation} 
The dual space  $\mathcal{B}_{\mathcal{N}}^*$ is again a Banach space.
Moreover, it follows from Proposition \ref{prop: BN is isometic isomorphic to quotient space} that the space $\mathcal{S}$ is a pre-dual space of $\mathcal{B}_{\mathcal{N}}$, that is, 
$(\mathcal{B}_{\mathcal{N}})_*=\mathcal{S}.$
We remark that the dual bilinear form on $\mathcal{B}_{\mathcal{N}}\times\mathcal{S}$ is given by
\begin{equation}\label{dual bilinear on BNS}
\langle f_{\mu}, g\rangle_{\mathcal{S}}=\langle\mu, g\rangle_{C_0(\Theta)}, \ \mbox{for}\ f_{\mu}\in\mathcal{B}_{\mathcal{N}},\ g\in \mathcal{S}.
\end{equation}
According to Proposition \ref{prop: BN is isometic isomorphic to quotient space}, the space $\mathcal{S}$ is the pre-dual space of $\mathcal{B}_{\mathcal{N}}$, that is,
$\mathcal{S}^*=\mathcal{B}_{\mathcal{N}}$.
Thus, we obtain that $\mathcal{S}^{**}=\mathcal{B}_{\mathcal{N}}^*$.
It is well-known (for example, see \cite{conway2019course}) that $\mathcal{S}\subseteq \mathcal{S}^{**}$ in the sense of isometric embedding.
Hence,   $\mathcal{S}\subseteq \mathcal{B}_{\mathcal{N}}^*$ and there holds 
\begin{equation}\label{natural-map-predual}
\langle g,f_{\mu}\rangle_{\mathcal{B}_{\mathcal{N}}}=\langle f_{\mu},g\rangle_{\mathcal{S}},
\ \mbox{for all} \ f_{\mu}\in \mathcal{B}_{\mathcal{N}} \ \mbox{and all} \ g\in \mathcal{S}.
\end{equation}

We now turn to establishing that $\mathcal{B}_{\mathcal{N}}$ is a vector-valued RKBS on $\mathbb{R}^s$.

\begin{theorem}\label{theorem: BN vector valued RKBS}
Let $\Theta$ be the parameter space defined by \eqref{Def:Theta}, and $\rho:\Theta\to\mathbb{R}^+$ be a continuous function on $\Theta$. If  
for each $x\in \mathbb{R}^s$ and $k\in\mathbb{N}_t$, the function $\mathcal{N}_k({x},\cdot)\rho(\cdot)$ belongs to $C_0(\Theta)$, then the Banach space $\mathcal{B}_{\mathcal{N}}$ defined by \eqref{banach space DNN} endowed with the norm \eqref{banach space norm DNN} is a vector-valued RKBS on $\mathbb{R}^s$.
\end{theorem}
\begin{proof}
According to Proposition \ref{prop: component wise for vector-valued RKBS} with $X:=\mathbb{R}^s$, it suffices to prove that for each $x\in\mathbb{R}^s$, $k\in\mathbb{N}_t$, there exists a positive constant $C_{x,k}$ such that
\begin{equation}\label{proof fmuk leq Cxk f norm}
    |f_\mu^k(x)|\leq C_{x,k}\|f_\mu\|_{\mathcal{B}_{\mathcal{N}}}, \ \mbox{for all}\ f_{\mu}\in\mathcal{B}_\mathcal{N}.
\end{equation}
To this end, for any $f_{\mu}\in\mathcal{B}_\mathcal{N}$, we obtain from definition \eqref{Def:f_mu^k} of $f_\mu^k$ that 
\begin{equation}\label{f_mu^k}
|f_\mu^k({x})|\leq \|\mathcal{N}_k({x},\cdot)\rho(\cdot)\|_\infty\|\nu\|_{\mathrm{TV}},
\end{equation}
for any $\nu\in\mathcal{M}(\Theta)$ satisfying $f_\nu=f_\mu$.  
By taking infimum of both sides of inequality \eqref{f_mu^k} over $\nu\in\mathcal{M}(\Theta)$ satisfying $f_\nu=f_\mu$ and employing definition \eqref{banach space norm DNN}, we obtain that 
$$
|f_\mu^k({x})|\leq\|\mathcal{N}_k({x},\cdot)\rho(\cdot)\|_\infty\|f_\mu\|_{\mathcal{B}_{\mathcal{N}}}.
$$
Letting $C_{x,k}:=\|\mathcal{N}_k({x},\cdot)\rho(\cdot)\|_\infty$, we get inequality \eqref{proof fmuk leq Cxk f norm}.
\end{proof}

Next, we identify 
the reproducing kernel of the vector-valued RKBS $\mathcal{B}_{\mathcal{N}}$. According to Proposition \ref{existence of reproducing kernel}, the existence of the reproducing kernel requires to characterize the $\delta$-dual space of $\mathcal{B}_{\mathcal{N}}$.  
We note that the $\delta$-dual space  $\mathcal{B}'_{\mathcal{N}}$ is the closure of  
\begin{equation*}\label{delta dual for BN}
\Delta:=\mathrm{span}\{\delta_{x,k}:x\in\mathbb{R}^s,\ k\in\mathbb{N}_t\},
\end{equation*}
in the norm topology \eqref{norm of BN*} of $\mathcal{B}_{\mathcal{N}}^*$. 
We will show that $\Delta$ is isometrically isomorphic to  
$$
\mathbb{S}:=\mathrm{span}\{\mathcal{N}_k({x},\cdot)\rho(\cdot):{x}\in\mathbb{R}^s,\ k\in\mathbb{N}_t\},
$$
a subspace of $\mathcal{S}$.   To this end, we introduce a mapping $\Psi:\Delta\to \mathbb{S}$ by 
\begin{equation}\label{isometric isomorphism_delta_KX}   \Psi\left(\sum_{j\in\mathbb{N}_m}\alpha_{j}\delta_{{x}_j,k_j}\right):= \sum_{j\in\mathbb{N}_m}\alpha_{j} \mathcal{N}_{k_j}({x}_j,\cdot)\rho(\cdot),
\end{equation}
for all $m\in\mathbb{N}$, $\alpha_{j}\in\mathbb{R}$, ${x}_j\in\mathbb{R}^s$, $k_j\in\mathbb{N}_t$, and $j\in\mathbb{N}_m$. 

\begin{lemma}\label{Lemma:isometric isomorphism_delta_KX}
The mapping $\Psi$ defined by \eqref{isometric isomorphism_delta_KX} is an isometric isomorphism between $\Delta$ and $\mathbb{S}$. 
\end{lemma}
\begin{proof}
We first prove that $\Psi$ is an isometry, that is, $\left\|\ell\right\|_{\mathcal{B}_{\mathcal{N}}^*}=\left\|\Psi(\ell)\right\|_{\infty}$, for all $\ell\in\Delta$. Let  $\ell$ be an arbitrary element of $\Delta$. Then there exist $m\in\mathbb{N}$, $\alpha_{j}\in\mathbb{R}$, ${x}_j\in\mathbb{R}^s$, $k_j\in\mathbb{N}_t$, and $j\in\mathbb{N}_m$ such that $\ell=\sum_{j\in\mathbb{N}_m}\alpha_{j}\delta_{{x}_j,k_j}$. By definition \eqref{norm of BN*} and the definition of the functionals $\delta_{x_j,k_j}$, $j\in\mathbb{N}_m$, we have that 
       \begin{equation}\label{proof norm of ell}
        \|\ell\|_{\mathcal{B}_{\mathcal{N}}^*}=\sup\left\{\left|\sum_{j\in\mathbb{N}_m}\alpha_{j} f_\mu^{k_j}(x_j)\right|:\left\|f_\mu\right\|_{\mathcal{B}_{\mathcal{N}}}=1, f_\mu\in\mathcal{B}_{\mathcal{N}}\right\}.
    \end{equation}
    We next compute $\|\Psi(\ell)\|_{\infty}$.
    By noting that $\Psi(\ell)\in \mathcal{S}$ and $\mathcal{S}^*=\mathcal{B}_{\mathcal{N}}$, we have that 
    $$
    \|\Psi(\ell)\|_\infty=\sup\left\{\left|\langle f_\mu, \Psi(\ell)\rangle_{\mathcal{S}}\right|:\left\|f_\mu\right\|_{\mathcal{B}_{\mathcal{N}}}=1, f_\mu\in\mathcal{B}_{\mathcal{N}}\right\}.
    $$
    Substituting equation \eqref{dual bilinear on BNS} with $g:=\Psi(\ell)$ into the right-hand side of the above equation, we get that 
    \begin{equation}\label{norm of psi ell}
        \|\Psi(\ell)\|_\infty=\sup\left\{\left|\langle \mu, \Psi(\ell)\rangle_{C_0(\Theta)}\right|:\left\|f_\mu\right\|_{\mathcal{B}_{\mathcal{N}}}=1, f_\mu\in\mathcal{B}_{\mathcal{N}}\right\}.
    \end{equation}
    According to definition \eqref{isometric isomorphism_delta_KX} of $\Psi$, there holds for any $f_{\mu}\in\mathcal{B}_{\mathcal{N}}$ that
    $$
    \langle \mu, \Psi(\ell)\rangle_{C_0(\Theta)}=\sum_{j\in\mathbb{N}_m}\alpha_{j} \langle \mu,\mathcal{N}_{k_j}({x}_j,\cdot)\rho(\cdot)\rangle_{C_0(\Theta)}.
    $$
    This together with definition \eqref{Def:f_mu^k} yields that 
    $
    \langle \mu, \Psi(\ell)\rangle_{C_0(\Theta)}=\sum_{j\in\mathbb{N}_m}\alpha_{j} f_{\mu}^{k_j}(x_j).
    $
    Substituting this equation into the right-hand side of \eqref{norm of psi ell} leads to
    \begin{equation}\label{proof: infinity normn of psi(ell) final form}
       \|\Psi(\ell)\|_\infty=\sup\left\{\left|\sum_{j\in\mathbb{N}_m}\alpha_{j} f_{\mu}^{k_j}(x_j)\right|:\left\|f_\mu\right\|_{\mathcal{B}_{\mathcal{N}}}=1, f_\mu\in\mathcal{B}_{\mathcal{N}}\right\}. 
    \end{equation}
    Comparing \eqref{proof norm of ell} and \eqref{proof: infinity normn of psi(ell) final form}, we obtain that 
    $   \|\ell\|_{\mathcal{B}_{\mathcal{N}}^*}=\|\Psi(\ell)\|_\infty$ and hence, $\Psi$ is an isometry between $\Delta$ and $\mathbb{S}$. The isometry of $\Psi$ further implies its injectivity.  Moreover, $\Psi$ is linear and surjective. Thus, $\Psi$ is  bijective. Therefore, $\Psi$ is an isometric isomorphism between $\Delta$ and $\mathbb{S}$. 
\end{proof}

The isometrically isomorphic relation between $\Delta$ and $\mathbb{S}$ is preserved after completing them. We state this result in the following lemma without proof.

\begin{lemma}\label{isometric-isomorphism-after-completion}
    Suppose that $A$ and $B$ are Banach spaces with norms $\|\cdot\|_A$ and $\|\cdot\|_B$, respectively. Let $A_0$ and $B_0$ be dense subsets of $A$ and $B$, respectively. If $A_0$ is isometrically isomorphic to $B_0$, then $A$ is isometrically isomorphic to $B$.
\end{lemma}

Lemma \ref{isometric-isomorphism-after-completion} may be obtained by applying Theorem 1.6-2 in \cite{kreyszig1991introductory}.
With the help of Lemmas \ref{Lemma:isometric isomorphism_delta_KX} and \ref{isometric-isomorphism-after-completion}, we identify in the following theorem the reproducing kernel for the RKBS $\mathcal{B}_{\mathcal{N}}$.

\begin{theorem}\label{theorem: Kernel of BN}
Let $\Theta$ be the parameter space defined by \eqref{Def:Theta}, and $\rho:\Theta\to\mathbb{R}^+$ be a continuous function on $\Theta$. Suppose that  
for each $x\in \mathbb{R}^s$ and $k\in\mathbb{N}_t$, the function $\mathcal{N}_k({x},\cdot)\rho(\cdot)$ belongs to $C_0(\Theta)$. If the vector-valued RKBS $\mathcal{B}_{\mathcal{N}}$ is defined by \eqref{banach space DNN} with the norm \eqref{banach space norm DNN}, then the vector-valued function 
\begin{equation}\label{Kernel}
\mathcal{K}(x,\theta):=\mathcal{N}(x,\theta)\rho(\theta), \ \ \mbox{for}\ \ (x,\theta)\in \mathbb{R}^s\times\Theta,
\end{equation}
is the reproducing kernel for space $\mathcal{B}_{\mathcal{N}}$.
\end{theorem}

\begin{proof}
We employ Proposition \ref{existence of reproducing kernel} with $X:=\mathbb{R}^s$ and $X':=\Theta$ to establish that the function $\mathcal{K}$ defined by \eqref{Kernel} is the reproducing kernel of space $\mathcal{B}_{\mathcal{N}}$. According to Lemma \ref{Lemma:isometric isomorphism_delta_KX}, $\Delta$ is isometrically isomorphic to $\mathbb{S}$. Since $\mathcal{B}_{\mathcal{N}}'$ and $\mathcal{S}$ are the completion of $\Delta$ and $\mathbb{S}$, respectively,  by Lemma \ref{isometric-isomorphism-after-completion}, we conclude that 
the $\delta$-dual space $\mathcal{B}'_{\mathcal{N}}$ of $\mathcal{B}_{\mathcal{N}}$ is isometrically isomorphic to $\mathcal{S}$, which is a Banach space of functions from $\Theta$ to $\mathbb{R}$. 
Hence, Proposition \ref{existence of reproducing kernel} ensures that there exists a unique reproducing kernel for $\mathcal{B}_{\mathcal{N}}$. 

We next verify that the vector-valued function $\mathcal{K}$ defined by \eqref{Kernel} is the reproducing kernel for $\mathcal{B}_{\mathcal{N}}$. It follows from definitions \eqref{subspace of C0} and \eqref{Kernel} that for each $x\in\mathbb{R}^s$ and each $k\in\mathbb{N}_t$,  $\mathcal{K}_k(x,\cdot):=\mathcal{N}_k(x,\cdot)\rho(\cdot)\in \mathcal{S}$. The space $\mathcal{S}$, guaranteed by Proposition \ref{prop: BN is isometic isomorphic to quotient space}, is a pre-dual space of $\mathcal{B}_{\mathcal{N}}$. Hence, by equation \eqref{natural-map-predual} with $g:=\mathcal{K}_k(x,\cdot)$, we obtain for each $x\in\mathbb{R}^s$,  $k\in\mathbb{N}_t$ that 
$$
\langle\mathcal{K}_k(x,\cdot), f_{\mu}\rangle_{\mathcal{B}_{\mathcal{N}}}=\langle f_{\mu},\mathcal{K}_k(x,\cdot)\rangle_{\mathcal{S}},\ \mbox{for all}\  f_{\mu}\in\mathcal{B}_{\mathcal{N}}.
$$
Substituting equation \eqref{dual bilinear on BNS} with $g:=\mathcal{K}_k(x,\cdot)$ into the right-hand side of the above equation leads to 
$$
\langle\mathcal{K}_k(x,\cdot), f_{\mu}\rangle_{\mathcal{B}_{\mathcal{N}}}=\langle\mu, \mathcal{K}_k(x,\cdot)\rangle_{C_0(\Theta)},\ \mbox{for all}\  f_{\mu}\in\mathcal{B}_{\mathcal{N}}.
$$
This together with definitions \eqref{DualBilinearForm}, \eqref{Kernel}  and \eqref{Def:f_mu^k} implies the reproducing property 
\begin{equation*}\label{proof reproducing property}
\langle\mathcal{K}_k(x,\cdot), f_{\mu}\rangle_{\mathcal{B}_{\mathcal{N}}}=f_{\mu}^k(x),\ \mbox{for all}\  f_{\mu}\in\mathcal{B}_{\mathcal{N}}.
\end{equation*}
Consequently, $\mathcal{K}$ is the reproducing kernel of $\mathcal{B}_{\mathcal{N}}$. 
\end{proof}

The reproducing kernel defined by \eqref{Kernel} in Theorem \ref{theorem: Kernel of BN} is an asymmetric kernel in contrast to the reproducing kernel in an RKHS, which is always symmetric. This asymmetry allows one variable of the kernel function to serve as the input variable and the other as the parameter variable, making it particularly suitable for encoding relationships in deep learning models. To the best of our knowledge, Theorem \ref{theorem: Kernel of BN}, even when restricted to shallow networks, presents a novel result. This reproducing kernel corresponds to the Banach space $\mathcal{S}$ of functions on $\Theta$, which is isometrically isomorphic to $\mathcal{B}_{\mathcal{N}}'$. In the next section, we will show that solutions to deep learning models can be expressed as kernel expansions based on training data. Therefore, this reproducing kernel represents the essential structure we seek for understanding deep learning representations.

We are ready to prove that the vector space $\mathcal{B}_{\mathbb{W}}$, defined by equation \eqref{space B Delta}, is weakly${}^*$ dense in the vector-valued RKBS $\mathcal{B}_{\mathcal{N}}$. For this purpose, we recall the concept of the weak${}^*$ topology. Let $B$ be a Banach space. The weak${}^*$ topology of the dual space $B^{*}$ is the smallest topology for $B^{*}$ such that, for each $f\in B$, the linear functional $\nu \rightarrow\langle\nu, f\rangle_{B}$ on $B^*$ is continuous with respect to the topology. For a subset $M'$ of $B^*$, we denote by $\overline{M'}^{w^*}$ the closure of $M'$ in the weak$^*$ topology of $B^*$. We remark that the fact that the Banach space $\mathcal{B}_{\mathcal{N}}$ has a pre-dual space $\mathcal{S}$ makes it valid for $\mathcal{B}_{\mathcal{N}}$ to be equipped with the weak$^*$ topology, the topology of $\mathcal{S}^*$.

\begin{theorem}\label{theorem: H is linear subspace}
Let $\Theta$ be the parameter space defined by \eqref{Def:Theta}, and $\rho:\Theta\to\mathbb{R}^+$ be a continuous function on $\Theta$. Suppose that $\mathbb{W}$ is the width set defined by \eqref{width-set} and the set $\mathcal{B}_{\mathbb{W}}$ is defined by \eqref{space B Delta}. If for each $x\in \mathbb{R}^s$ and $k\in\mathbb{N}_t$, the function $\mathcal{N}_k({x},\cdot)\rho(\cdot)$ belongs to $C_0(\Theta)$, then $\mathcal{B}_{\mathbb{W}}$ is a subspace of $\mathcal{B}_{\mathcal{N}}$ and  
\begin{equation}\label{weak*Dense}
\overline{\mathcal{B}_{\mathbb{W}}}^{w^*}=\mathcal{B}_{\mathcal{N}}.
\end{equation}
\end{theorem}
\begin{proof}
It has been shown in Proposition \ref{prop: BM is a linear space} that $\mathcal{B}_{\mathbb{W}}$ is a vector space. We now show that $\mathcal{B}_{\mathbb{W}}$ is a subspace of $\mathcal{B}_{\mathcal{N}}$. For any $f\in{{{\mathcal{B}_{\mathbb{W}}}}}$, there exist $n\in\mathbb{N}$, $c_l\in\mathbb{R}$, $\theta_l\in\Theta$, $l\in\mathbb{N}_n$ such that $f=\sum_{l=1}^n c_l\mathcal{N}(\cdot,\theta_l)\rho(\theta_l)$. By choosing $\mu:=\sum_{l=1}^n c_l\delta_{\theta_l}$, we have that 
$\mu\in\mathcal{M}(\Theta)$. We then obtain from definition \eqref{Def:f_mu^k} that  
$$
f_{\mu}^k(x)=\sum_{l=1}^n c_l\mathcal{N}_k(x,\theta_l)\rho(\theta_l),\ \mbox{for all}\ x\in\mathbb{R}^s,\  k\in\mathbb{N}_t.
$$
This together with the representation of $f$ yields that $f=f_{\mu}$ and thus, $f\in\mathcal{B}_{\mathcal{N}}$. Consequently, we have that $\mathcal{B}_{\mathbb{W}}\subseteq\mathcal{B}_{\mathcal{N}}$.

It remains to prove equation \eqref{weak*Dense}. Proposition \ref{prop: BN is isometic isomorphic to quotient space} ensures that $\mathcal{S}^*=\mathcal{B}_{\mathcal{N}}$, in the sense of being isometrically isomorphic. Hence, $\mathcal{B}_{\mathbb{W}}$ is a subspace of the dual space of $\mathcal{S}$. It follows from Proposition 2.6.6 of \cite{megginson2012introduction} that $(^{\perp}\mathcal{B}_{\mathbb{W}})^{\perp}=\overline{\mathcal{B}_{\mathbb{W}}}^{w^*}$. It suffices to verify that $(^{\perp}\mathcal{B}_{\mathbb{W}})^{\perp}=\mathcal{B}_{\mathcal{N}}$. Due to definition \eqref{space B Delta} of $\mathcal{B}_{\mathbb{W}}$, $g\in  ^{\perp}\mathcal{B}_{\mathbb{W}}$ if and only if \begin{equation}\label{suff-ness-1}
    \langle \mathcal{N}(\cdot,\theta)\rho(\theta), g\rangle_{\mathcal{S}}=0,\ \mbox{for all}\ \theta\in\Theta.
\end{equation}
By equation \eqref{dual bilinear on BNS} with $f_{\mu}:=\mathcal{N}(\cdot,\theta)\rho(\theta)$ with $\mu=\delta_{\theta}$, equation \eqref{suff-ness-1} is equivalent to 
$\left<\delta_\theta, g\right>_{C_0(\Theta)}=0$, for all $\theta\in\Theta$,
which leads to $g(\theta)=0$ for all $\theta\in\Theta$. That is, $g=0$. Therefore, $^{\perp}\mathcal{B}_{\mathbb{W}}=\{0\}$. This together with the definition of annihilators leads to $(^{\perp}\mathcal{B}_{\mathbb{W}})^\perp=\mathcal{B}_{\mathcal{N}}$, which completes the proof of this theorem.
\end{proof}


We now return to the discussion on the relationship between the hypothesis space and the network layer width set $\mathbb{W}$, assuming a fixed network depth $D$ and the activation function $\sigma$. From the construction of the linear space $\mathcal{B}_{\mathbb{W}}$, we observe that $\mathcal{B}_{\mathbb{W}}$ corresponds to an infinitely wide block neural network, where each block represents a neural network with the layer width set $\mathbb{W}$.

First, we consider a shallow neural network (a neural network with only one hidden layer).  Clearly, the linear combination of any $n$ neural networks with the layer width set $\mathbb{W}$ forms a fully connected neural network with the layer width set $n\mathbb{W}$. Since the linear space $\mathcal{B}_{\mathbb{W}}$ consists of all linear combinations of an arbitrary $n$ networks with the parameter $\theta_l\in \Theta_\mathbb{W}$, it includes  fully connected neural networks with potentially infinite width. Consequently, as a linear space,  $\mathcal{B}_{\mathbb{W}}$, and thus $\mathcal{B}_{\mathcal{N}}$, is independent of the specific layer width set $\mathbb{W}$. However, the constructed hypothesis space is a Banach space, which not only possesses an algebraic structure but is also equipped with a topological structure induced by a norm. As seen from its definition, the norm of $\mathcal{B}_{\mathcal{N}}$ is determined by the measure on $\Theta_\mathbb{W}$, making it inherently depends on the network layer width set $\mathbb{W}$. This dependence implies that $\mathcal{B}_{\mathcal{N}}$, as a Banach space, varies with the layer width set. It is worth noting that different norms lead to different reproducing kernels for $\mathcal{B}_{\mathcal{N}}$. We have demonstrated that the reproducing kernel of $\mathcal{B}_{\mathcal{N}}$ is given by the product of a neural network with layer width set $\mathbb{W}$ and a certain weight function. In other words, as a Banach space, $\mathcal{B}_{\mathcal{N}}$ has a reproducing kernel that is inherently related to the network layer width set $\mathbb{W}$. In summary, even for a shallow neural network, the hypothesis space  $\mathcal{B}_{\mathcal{N}}$ remains fundamentally dependent on the network layer width set $\mathbb{W}$.

Next, we consider a deep neural network (a neural network with multiple hidden layers). While the linear combination of any $n$ neural networks with the layer width set $\mathbb{W}$ still results in a neural network with the network layer width set $n\mathbb{W}$, the neurons in different network blocks and layers remain unconnected. In other words, the resulting neural network with the layer width set $n\mathbb{W}$ is not fully connected. This occurs because the linearization operation is applied only to the output layer. Again, since the linear space $\mathcal{B}_{\mathbb{W}}$ consists of all linear combinations of an arbitrary $n$ networks with parameters $\theta_l\in \Theta_\mathbb{W}$, it does not contain a fully connected neural network with infinite width. Notably, for the same network layer width, fully connected networks generally exhibit a richer structure than partially connected networks. For example, in the case of  ReLU networks, fully connected networks allow for piecewise linear polynomials with a greater number of partitioning nodes. Based on the above discussion, even as a linear space, $\mathcal{B}_{\mathbb{W}}$ and correspondingly $\mathcal{B}_{\mathcal{N}}$, remains fundamentally dependent on the network layer width set $\mathbb{W}$. Similar to shallow neural networks, the topological structure and the reproducing kernel of the hypothesis space for deep neural networks also depend on the layer width.

To conclude this section, we summarize the key properties of the space $\mathcal{B}_{\mathcal{N}}$ established in 
Theorems \ref{theorem: BN vector valued RKBS}, \ref{theorem: Kernel of BN} and \ref{theorem: H is linear subspace}:

(i) $\mathcal{B}_{\mathcal{N}}$ is a vector-valued RKBS.

(ii) The vector-valued function $\mathcal{K}$ defined in  \eqref{Kernel}
serves as the reproducing kernel for $\mathcal{B}_{\mathcal{N}}$.

(iii) $\mathcal{B}_{\mathcal{N}}$ is the weak* completion of the vector space $\mathcal{B}_{\mathbb{W}}$.

\noindent
These desirable properties of $\mathcal{B}_{\mathcal{N}}$ make it a natural choice as the hypothesis space for deep learning.
Accordingly, we consider the following learning model:  
\begin{equation}\label{LearningMethodinRKBS}
    \inf\{\mathcal{L}(f_{\mu},\mathbb{D}_m): f_{\mu}\in\mathcal{B}_{\mathcal{N}}\}.
\end{equation}
Let $\mathcal{N}_{\mathcal{B}_\mathcal{N}}$ denote the neural network learned from the model \eqref{LearningMethodinRKBS}. Then, according to \eqref{Comparison-of-three-models}, we have
\begin{equation*}
    \mathcal{L}(\mathcal{N}_{\mathcal{B}_\mathcal{N}},\mathbb{D}_m)\leq \mathcal{L}(\mathcal{N}_{\mathcal{B}_\mathbb{W}},\mathbb{D}_m)
    \leq \mathcal{L}(\mathcal{N}_{\mathcal{A}_\mathbb{W}},\mathbb{D}_m).
\end{equation*}
Although the learning model \eqref{LearningMethodinRKBS},  like \eqref{GeneralLearningMethod}, operates in an infinite dimension (unlike \eqref{BasicLearningMethod-equivalent}, which is finite-dimensional), we will demonstrate in the next section that a solution to \eqref{LearningMethodinRKBS} lays within a finite-dimensional manifold determined by the kernel $\mathcal{K}$ and a given dataset.


\section{Representer Theorems for Learning Solutions}\label{section: representer theorems}
In this section, we consider learning a target function in $\mathcal{B}_{\mathcal{N}}$ from the sampled dataset $\mathbb{D}_m$ defined by \eqref{Dataset}. Learning such a function is an ill-posed problem, often prone to overfitting. To address this, instead of solving the learning model \eqref{LearningMethodinRKBS} directly, we focus on a related regularization problem and the MNI problem in the RKBS $\mathcal{B}_{\mathcal{N}}$. The goal of this section is to establish the representer theorems for solutions to these two learning models.


We begin by describing the regularized learning problem in the vector-valued RKBS $\mathcal{B}_{\mathcal{N}}$.  For the dataset $\mathbb{D}_m$ defined by \eqref{Dataset}, we define the set $\mathcal{X}:=\{x_j:j\in\mathbb{N}_m\}$ and the matrix $\mathbf{Y}:=[y_j^k:k\in\mathbb{N}_t,j\in\mathbb{N}_{m}]\in\mathbb{R}^{t\times m}$, where for each $j\in\mathbb{N}_m$, $y_j^k$, $k\in\mathbb{N}_t$, are the components of the vector $y_j$. 
We introduce an operator $\mathbf{I}_\mathcal{X}:{\mathcal{B}_{\mathcal{N}}} \rightarrow \mathbb{R}^{t\times m}$ defined by 
\begin{equation}\label{L YES DC}
\mathbf{I}_\mathcal{X}(f_\mu):=\left[f^k_{\mu}(x_j): k\in\mathbb{N}_t,j\in\mathbb{N}_{m}\right],\ \mbox{for}\ f_\mu\in\mathcal{B}_{\mathcal{N}}.
\end{equation} 
We choose a loss function $\mathcal{Q}: \mathbb{R}^{t\times m} \rightarrow \mathbb{R}_{+}:=[0,+\infty)$ and define the empirical loss as
\begin{equation}\label{loss:An example}
    \mathcal{L}(f_{\mu},\mathbb{D}_m):=\mathcal{Q}(\mathbf{I}_{\mathcal{X}}(f_\mu)-\mathbf{Y}),\ \mbox{for}\ f_{\mu}\in\mathcal{B}_{\mathcal{N}}.
\end{equation}
An example of the loss function $\mathcal{Q}({\mathbf{M}})$ could be a norm of the matrix $\mathbf{M}$.
The proposed regularization problem is formed by adding a regularization
term $\lambda\| f_\mu\|_{{\mathcal{B}_{\mathcal{N}}}}$ to the data fidelity term $\mathcal{Q}(\mathbf{I}_{\mathcal{X}}(f_{\mu})-\mathbf{Y})$. That is, 
\begin{equation}\label{eq: regularization problem RKBS B measure M(X)}
    \inf \left\{\mathcal{Q}(\mathbf{I}_{\mathcal{X}}(f_\mu)-\mathbf{Y})+\lambda\| f_\mu\|_{{\mathcal{B}_{\mathcal{N}}}}:  f_\mu\in {\mathcal{B}_{\mathcal{N}}}\right\},
\end{equation}
where $\lambda$ is a positive regularization parameter. 

Next, we establish the existence of a solution to the regularization problem \eqref{eq: regularization problem RKBS B measure M(X)}, which follows directly from Proposition 40 in \cite{wang2021representer}.


\begin{proposition}\label{Existence-of-Solution}
Suppose that $m$ distinct points $x_j\in\mathbb{R}^s$, $j\in\mathbb{N}_m$, and $\mathbf{Y}\in\mathbb{R}^{t\times m}$ are given.  
If $\lambda>0$ and the loss function $\mathcal{Q}$ is lower semi-continuous on $\mathbb{R}^{t\times m}$, then the regularization problem \eqref{eq: regularization problem RKBS B measure M(X)} has at least one solution. 
\end{proposition}
%

It is known that regularization problems are closely related to MNI problems (see, for example, \cite{wang2021representer}). The MNI problem seeks a vector-valued function $f_{\mu}$ in $\mathcal{B}_{\mathcal{N}}$ with the smallest norm that
satisfies the interpolation condition $f_{\mu}(x_j)=y_j$, $j\in\mathbb{N}_m$. Formally, the MNI problem is given by 
\begin{equation}\label{MNI-original}
\inf
\{\|f_{\mu}\|_{\mathcal{B}_{\mathcal{N}}}: f_{\mu}(x_j)=y_j, f_{\mu}\in\mathcal{B}_{\mathcal{N}}, j\in\mathbb{N}_m\}.
\end{equation}
To reformulate this problem, we define the subset $\mathcal{M}_{\mathcal{X},\mathbf{Y}}$ of $\mathcal{B}_{\mathcal{N}}$ as
\begin{equation*}\label{hyperplane YES DC}
    \mathcal{M}_{\mathcal{X},\mathbf{Y}}:=\{f_\mu \in {\mathcal{B}_{\mathcal{N}}}: \mathbf{I}_{\mathcal{X}}(f_\mu)=\mathbf{Y}\}.
\end{equation*}
Thus, the MNI problem \eqref{MNI-original} can be equivalently written as
\begin{equation}\label{MNI in RKBS B measure M(X)}
    \inf \left\{\left\|f_\mu\right\|_{{\mathcal{B}_{\mathcal{N}}}}: f_\mu \in  \mathcal{M}_{\mathcal{X},\mathbf{Y}}\right\}.
\end{equation}

A solution to the MNI problem \eqref{MNI in RKBS B measure M(X)} exists if and only if
$\mathcal{M}_{\mathcal{X},\mathbf{Y}}$ is nonempty. 
The non-emptiness of  $\mathcal{M}_{\mathcal{X},\mathbf{Y}}$ concerns whether any function in $\mathcal{B}_{\mathcal{N}}$ can interpolate the given dataset $\mathbb{D}_m$. For the remainder of this paper, we assume $\mathcal{M}_{\mathcal{X},\mathbf{Y}}$ is nonempty.
Since $\mathcal{K}$, defined in \eqref{Kernel}, is the reproducing kernel of $\mathcal{B}_{\mathcal{N}}$, the operator $\mathbf{I}_{\mathcal{X}}$ from \eqref{L YES DC} can be expressed as
$$
\mathbf{I}_{\mathcal{X}}(f_\mu)=\left[\langle \mathcal{K}_k(x_j,\cdot), f_{\mu}\rangle_{\mathcal{B}_{\mathcal{N}}}: k\in\mathbb{N}_t,j\in\mathbb{N}_{m}\right],\ \mbox{for}\ f_\mu\in\mathcal{B}_{\mathcal{N}}.
$$
This shows that the MNI problem \eqref{MNI in RKBS B measure M(X)} imposes $tm$ interpolation constraints induced by the functionals in the set
$
\mathbb{K}_\mathcal{X}:=\{\mathcal{K}_k(x_j,\cdot):  j\in\mathbb{N}_m,\ k\in\mathbb{N}_t\}.
$
The following result regarding the existence of a solution to the MNI problem \eqref{MNI in RKBS B measure M(X)}
is a direct consequence of Proposition 1 of \cite{wang2021representer}.
\begin{proposition}\label{Existence-of-Solution-MNI}
Suppose that $m$ distinct points $x_j\in\mathbb{R}^s$, $j\in\mathbb{N}_m$, and $\mathbf{Y}\in\mathbb{R}^{t\times m}$ are given.  
If the functionals in $\mathbb{K}_\mathcal{X}$ are linearly independent in $\mathcal{S}$, then the MNI problem \eqref{MNI in RKBS B measure M(X)} has at least one solution. 
\end{proposition}
Indeed, the linear independence of the functionals in $\mathbb{K}_\mathcal{X}$ is a sufficient condition to ensure $\mathcal{M}_{\mathcal{X},\mathbf{Y}}$ is nonempty for any given $\mathbf{Y}\in\mathbb{R}^{t\times m}$. 
Therefore, without loss of generality, we assume that the functionals in $\mathbb{K}_\mathcal{X}$ are linearly independent throughout the remainder of this paper. 


Next, we establish a representer theorem for the solution to the MNI problem \eqref{MNI in RKBS B measure M(X)}. This follows from the explicit, data-dependent representer theorem for the MNI problem in a general Banach space setting, established in our recent paper \cite{wang2023sparse}. To apply this result to our current setting, we review the necessary notions of convex analysis in Appendix A and restate the relevant theorem in Appendix B. Specifically, we prepare to apply Lemma \ref{lemma: representer for MNI} from Appendix B to the MNI problem \eqref{MNI in RKBS B measure M(X)}.
To this end, we introduce the subspace 
\begin{equation}\label{V_span_kernel}
\mathcal{V}_{\mathcal{N}}:=\mathrm{span}\ \mathbb{K}_\mathcal{X}
\end{equation}
of $\mathcal{S}$, which has been shown to be a pre-dual space of $\mathcal{B}_{\mathcal{N}}$.
Let $\mathbb{S}_{\mathcal{X},\mathbf{Y}}$ denote the solution set of the MNI problem \eqref{MNI in RKBS B measure M(X)}. 

We now introduce the dual formulation of the MNI problem, given by
\begin{equation}\label{dual problem}
    \sup\left\{ \sum_{k\in\mathbb{N}_t}\sum_{j\in\mathbb{N}_m}c_{kj}y_j^k:\left\|\sum_{k\in\mathbb{N}_t}\sum_{j\in\mathbb{N}_m} c_{kj}\mathcal{K}_k(x_j,\cdot)\right\|_{\infty}=1, c_{kj}\in\mathbb{R},k\in\mathbb{N}_t,j\in\mathbb{N}_m\right\}.
\end{equation}
This is a finite-dimensional optimization problem that shares the same optimal value, denoted by ${C^*}$, as the MNI problem \eqref{MNI in RKBS B measure M(X)}. 
Using arguments similar to those in \cite{cheng2021minimum}, we establish the existence of a solution to the dual problem \eqref{dual problem}. As a convex optimization problem, it can be efficiently solved via linear programming, since it involves maximizing a linear function over a convex polytope \cite{ChengWangXu2023,cheng2021minimum}. Let $\hat{\mathbf{c}}:=[\hat{c}_{kj}:k\in\mathbb{N}_t,j\in\mathbb{N}_m]\in\mathbb{R}^{t\times m}$ be a solution to the dual problem \eqref{dual problem}. We then define
\begin{equation}\label{stage 1 hat g}
    \hat{g}(x):={C^*}\sum_{k\in\mathbb{N}_t}\sum_{j\in\mathbb{N}_m}\hat{c}_{kj}\mathcal{K}_k(x_j,x),\ x\in\mathbb{R}^s.
\end{equation}
As defined in Appendix A, let $\mathrm{ext}(A)$ denote the set of extreme points of a nonempty closed
convex set $A$. Additionally, let $\partial\|\cdot\|_B(f)$ represent the subdifferential of the norm function $\|\cdot\|_B$ in a Banach space at each $f\in B\backslash\{0\}$. 

\begin{theorem}\label{theorem: direct representer theorem for MNI in BN}
Suppose that $m$ distinct points $x_j\in\mathbb{R}^s$, $j\in\mathbb{N}_m$, and $\mathbf{Y}\in\mathbb{R}^{t\times m}\backslash\{\mathbf{0}\}$ are given, and the functionals in $\mathbb{K}_\mathcal{X}$ are linearly independent. 
Let $\hat g$ be  the function defined by \eqref{stage 1 hat g}.
Then for any extreme point $\hat f$ of the solution set $\mathbb{S}_{\mathcal{X},\mathbf{Y}}$ of the MNI problem \eqref{MNI in RKBS B measure M(X)}), there exist $\gamma_\ell\in\mathbb{R}$, $\ell\in\mathbb{N}_{tm}$, with  $\sum_{\ell\in\mathbb{N}_{tm}}\gamma_\ell=\|\hat g\|_{\infty}$ and $h_\ell\in\mathrm{ext}(\partial\|\cdot\|_\infty(\hat g))$, $\ell\in\mathbb{N}_{tm}$, such that 
    \begin{equation}\label{representer solution of finite linear combination of MNI in BN}
        \hat f(x)=\sum\limits_{\ell\in\mathbb{N}_{tm}}\gamma_\ell h_\ell(x),\ x\in\mathbb{R}^s.
    \end{equation}    
\end{theorem}
\begin{proof}
Proposition \ref{prop: BN is isometic isomorphic to quotient space} ensures that the vector-valued RKBS $\mathcal{B}_{\mathcal{N}}$ has the pre-dual space $\mathcal{S}$. Note that the functionals in $\mathbb{K}_\mathcal{X}$ belong to the pre-dual space $\mathcal{S}$ and are linearly independent. Moreover, since $\hat g$ is the function  defined by  \eqref{stage 1 hat g}, we have that $\hat g\in\mathcal{V}_{\mathcal{N}}$, and   according to Proposition 37 of \cite{wang2023sparse}, $\hat g$ satisfies the condition
\begin{equation}\label{non empty set: BN case}
(\|\hat g\|_{\infty}\partial\|\cdot\|_\infty(\hat g))\cap{\mathcal{M}}_{\mathcal{X},\mathbf{Y}}\neq\emptyset.
\end{equation} 
Hence, the hypothesis of Lemma \ref{lemma: representer for MNI} is satisfied. Then by Lemma \ref{lemma: representer for MNI}, we can represent any extreme point $\hat f$ of the solution set $\mathbb{S}_{\mathcal{X},\mathbf{Y}}$ of the MNI problem \eqref{MNI in RKBS B measure M(X)} as in equation \eqref{representer solution of finite linear combination of MNI in BN} for some $\gamma_\ell\in\mathbb{R}$, $\ell\in\mathbb{N}_{tm}$, with  $\sum_{\ell\in\mathbb{N}_{tm}}\gamma_\ell=\|\hat g\|_{\infty}$ and $h_\ell\in\mathrm{ext}(\partial\|\cdot\|_\infty(\hat g))$, $\ell\in\mathbb{N}_{tm}$.  
\end{proof}

Theorem \ref{theorem: direct representer theorem for MNI in BN} provides for each extreme point of the solution set of problem \eqref{MNI in RKBS B measure M(X)} an explicit, data-dependent representation by using the elements in $\mathrm{ext}(\partial\|\cdot\|_\infty(\hat g))$. Even more significantly, the essence of Theorem \ref{theorem: direct representer theorem for MNI in BN} is that although the MNI problem \eqref{MNI in RKBS B measure M(X)} is of infinite dimension, every extreme point of its solution set lays in a {\it finite} dimensional manifold spanned by $tm$ elements  $h_\ell\in\mathrm{ext}(\partial\|\cdot\|_\infty(\hat g))$.

As we have demonstrated earlier, the element $\hat g$ satisfying \eqref{non empty set: BN case} can be obtained by solving the dual problem \eqref{dual problem} of \eqref{MNI in RKBS B measure M(X)}.
Since $\hat{g}$ is an element in $\mathcal{S}$, the subdifferential $\partial\|\cdot\|_\infty(\hat g)$ is a subset of the space $\mathcal{B}_{\mathcal{N}}$, which is the dual space of $\mathcal{S}$. Notice that the subdifferential set $\partial\|\cdot\|_\infty(\hat g)$ may not be included in the space $\mathcal{B}_{\mathbb{W}}$  defined by \eqref{space B Delta} which is spanned by functions $\mathcal{K}(\cdot,\theta)$, $\theta\in\Theta$.
However, a learning solution in the vector-valued RKBS $\mathcal{B}_{\mathcal{N}}$ is expected to be represented by functions $\mathcal{K}(\cdot,\theta)$, $\theta\in\Theta$. For the purpose of obtaining a kernel representation for a solution to problem \eqref{MNI in RKBS B measure M(X)}, alternatively to problem \eqref{MNI in RKBS B measure M(X)}, we consider a closely related MNI problem in the measure space $\mathcal{M}(\Theta)$ and apply Lemma \ref{lemma: representer for MNI} to it. We then translate the resulting representer theorem for the MNI problem in $\mathcal{M}(\Theta)$ to that for problem \eqref{MNI in RKBS B measure M(X)}, by using the relation between the solutions of these two problems.

We now introduce the MNI problem in the measure space $\mathcal{M}(\Theta)$ with respect to the sampled dataset $\mathbb{D}_m$ and show the relation between its solution and a solution to problem \eqref{MNI in RKBS B measure M(X)}. By defining an operator $\widetilde{\mathbf{I}}_{\mathcal{X}}:\mathcal{M}(\Theta) \rightarrow \mathbb{R}^{t\times m}$ by
\begin{equation*}\label{tilde L on measure space}
\widetilde{\mathbf{I}}_{\mathcal{X}}(\mu):=\left[\langle \mathcal{K}_k(x_j,\cdot),\mu\rangle_{\mathcal{M}(\Theta)}: k\in\mathbb{N}_t, j \in \mathbb{N}_{m}\right],\ \mbox{for all}\ \mu\in\mathcal{M}(\Theta), 
\end{equation*} 
and introducing a subset $\widetilde{\mathcal{M}}_{\mathcal{X},\mathbf{Y}}$ of $\mathcal{M}(\Theta)$ as 
\begin{equation*}\label{hyperplane in measure space}
    \widetilde{\mathcal{M}}_{\mathcal{X},\mathbf{Y}}:=\left\{\mu\in\mathcal{M}(\Theta):\widetilde{\mathbf{I}}_{\mathcal{X}}(\mu)=\mathbf{Y}\right\},
\end{equation*}
we formulate the MNI problem in $\mathcal{M}(\Theta)$ as
\begin{equation}\label{MNI in measure space}
    \inf \left\{\left\|\mu\right\|_{\mathrm{TV}}: \mu \in \widetilde{\mathcal{M}}_{\mathcal{X},\mathbf{Y}}\right\}.
\end{equation}

The next proposition reveals the relation between the solutions of \eqref{MNI in RKBS B measure M(X)} and \eqref{MNI in measure space}. 

\begin{proposition}\label{prop: MNI solution for measure space is solution for RKBS}
Suppose that $m$ distinct points $x_j\in\mathbb{R}^s$, $j\in\mathbb{N}_m$, and $\mathbf{Y}\in\mathbb{R}^{t\times m}\backslash\{\mathbf{0}\}$ are given, and the functionals in $\mathbb{K}_\mathcal{X}$ are linearly independent. 
If ${\hat{\mu}}$ is a solution to the MNI problem \eqref{MNI in measure space}, then  
$f_{\hat\mu}(x):=\left[f_{\hat\mu}^k(x): k\in\mathbb{N}_t\right]^\top$, $x\in\mathbb{R}^s$,
with $f_{\hat\mu}^k$, $k\in\mathbb{N}_t$, defined as in  \eqref{Def:f_mu^k} with $\mu$ replaced by $\hat{\mu}$, is a solution to the MNI problem \eqref{MNI in RKBS B measure M(X)} 
and $\|f_{{\hat{\mu}}}\|_{\mathcal{B}_{\mathcal{N}}}=\|{\hat{\mu}}\|_{\mathrm{TV}}$. 
\end{proposition}
\begin{proof}
Note that the measure space $\mathcal{M}(\Theta)$ has the pre-dual space $C_0(\Theta)$. Hence, it follows from Proposition 1 of \cite{wang2021representer} that the linear independence of  the functionals in $\mathbb{K}_\mathcal{X}$ ensures the existence of a solution to problem \eqref{MNI in measure space}. Assume that ${\hat{\mu}}$ is a solution to problem \eqref{MNI in measure space}. We then obtain that ${\hat{\mu}}\in\widetilde{\mathcal{M}}_{\mathcal{X},\mathbf{Y}}$ and 
\begin{equation}\label{proof mu* is a solution}
        \|{\hat{\mu}}\|_{\mathrm{TV}}\leq\|\mu\|_{\mathrm{TV}},\ \text{for all }\mu\in\widetilde{\mathcal{M}}_{\mathcal{X},\mathbf{Y}}.
\end{equation} 
By equations \eqref{dual bilinear on BNS} and \eqref{natural-map-predual} with $g:=\mathcal{K}_k(x_j,\cdot)$, we have for each $k\in\mathbb{N}_t$ and each $j\in\mathbb{N}_{m}$ that
    \begin{equation}\label{verify interpolation condition}
    \langle \mathcal{K}_k(x_j,\cdot), f_{\mu}\rangle_{\mathcal{B}_{\mathcal{N}}}=\langle \mu,\mathcal{K}_k(x_j,\cdot)\rangle_{C_0(\Theta)}, \ \mbox{for all}\ \mu\in\mathcal{M}(\Theta).
    \end{equation}
    Note that $\mathcal{K}_k(x_j,\cdot)\in C_0(\Theta)$ can be viewed as a bounded linear functional on $\mathcal{M}(\Theta)$ and 
    $$\langle \mu,\mathcal{K}_k(x_j,\cdot)\rangle_{C_0(\Theta)}=\langle \mathcal{K}_k(x_j,\cdot),\mu\rangle_{\mathcal{M}(\Theta)}.
    $$ 
    Substituting the above equation into the right-hand of equation \eqref{verify interpolation condition} leads to 
    $$
    \langle \mathcal{K}_k(x_j,\cdot), f_{\mu}\rangle_{\mathcal{B}_{\mathcal{N}}}=\langle \mathcal{K}_k(x_j,\cdot),\mu\rangle_{\mathcal{M}(\Theta)}, \ \mbox{for all}\ \mu\in\mathcal{M}(\Theta).
    $$
    This implies that $        \mathbf{I}_{\mathcal{X}}(f_{\mu})=\widetilde{\mathbf{I}}_{\mathcal{X}}(\mu)$ for all $\mu\in\mathcal{M}(\Theta)$. As a result,  
    $\mu\in\widetilde{\mathcal{M}}_{\mathcal{X},\mathbf{Y}}$ if and only if 
 $f_{\mu}\in\mathcal{M}_{\mathcal{X},\mathbf{Y}}$. Since ${\hat{\mu}}\in\widetilde{\mathcal{M}}_{\mathcal{X},\mathbf{Y}}$, we get that $f_{\hat\mu}\in\mathcal{M}_{\mathcal{X},\mathbf{Y}}$. It suffices to verify that 
  \begin{equation*}
\|f_{{\hat{\mu}}}\|_{\mathcal{B}_{\mathcal{N}}}\leq\|f_\mu\|_{\mathcal{B}_{\mathcal{N}}},\quad\text{for all }f_{\mu}\in{\mathcal{M}}_{\mathcal{X},\mathbf{Y}}.
    \end{equation*}
    Let $f_{\mu}$ be an arbitrary element in ${\mathcal{M}}_{\mathcal{X},\mathbf{Y}}$. For any $\nu\in\mathcal{M}(\Theta)$ satisfying $f_{\mu}=f_{\nu}$, there holds
    $f_{\nu}\in{\mathcal{M}}_{\mathcal{X},\mathbf{Y}}$. Thus,  $\nu\in\widetilde{\mathcal{M}}_{\mathcal{X},\mathbf{Y}}$. It follows from inequality \eqref{proof mu* is a solution} that
\begin{equation}\label{Relation:TV-Norm}
     \|{\hat{\mu}}\|_{\mathrm{TV}}\leq\|\nu\|_{\mathrm{TV}}. 
\end{equation}
By taking infimum of both sides of the inequality \eqref{Relation:TV-Norm} over $\nu\in\mathcal{M}(\Theta)$ satisfying $f_\mu=f_\nu$ and noting the definition \eqref{banach space norm DNN} of the  norm $\|f_\mu\|_{\mathcal{B}_{\mathcal{N}}}$, we get that  
\begin{equation}\label{proof mu* leq fnu}
        \|{\hat{\mu}}\|_{\mathrm{TV}}\leq\|f_\mu\|_{\mathcal{B}_{\mathcal{N}}}. 
\end{equation}
Again by the definition \eqref{banach space norm DNN} of the  norm $\|f_{\hat\mu}\|_{\mathcal{B}_{\mathcal{N}}}$, we obtain that
\begin{equation}\label{proof fmu* leq mu*}    \|f_{{\hat{\mu}}}\|_{\mathcal{B}_{\mathcal{N}}}\leq\|{\hat{\mu}}\|_{\mathrm{TV}}.
\end{equation}
   Combining inequalities \eqref{proof mu* leq fnu} with \eqref{proof fmu* leq mu*}, we conclude that  $\|f_{{\hat{\mu}}}\|_{\mathcal{B}_{\mathcal{N}}}\leq\|f_{\mu}\|_{\mathcal{B}_{\mathcal{N}}}$. Therefore, $f_{{\hat{\mu}}}$ is a solution to the MNI problem \eqref{MNI in RKBS B measure M(X)}. 
   Moreover, by taking $\mu={\hat{\mu}}$ in \eqref{proof mu* leq fnu}, we get that $\|{\hat{\mu}}\|_{\mathrm{TV}}\leq\|f_{{\hat{\mu}}}\|_{\mathcal{B}_{\mathcal{N}}}$. This together with inequality \eqref{proof fmu* leq mu*} leads to $\|f_{{\hat{\mu}}}\|_{\mathcal{B}_{\mathcal{N}}}=\|{\hat{\mu}}\|_{\mathrm{TV}}$. 
\end{proof}

We next derive a representer theorem for a solution to problem \eqref{MNI in measure space} by employing Lemma \ref{lemma: representer for MNI}. 
Applying Lemma \ref{lemma: representer for MNI} to problem \eqref{MNI in measure space} requires the representation of the extreme points of the subdifferential set $\partial\|\cdot\|_\infty(g)$ for any nonzero $g\in C_0(\Theta)$. Here, the subdifferential set $\partial\|\cdot\|_\infty(g)$ is a subset of the measure space $\mathcal{M}(\Theta)$. 
For each $g\in C_0(\Theta)$, let $\Theta(g)$ denote the subset of $\Theta$ where the function $g$ attains its maximum norm $\|g\|_\infty$, that is, 
\begin{equation}\label{def: infinity set for function new}
    \Theta(g):=\left\{\theta\in \Theta:|g(\theta)|=\|g\|_\infty\right\}.
\end{equation}
For each $g\in C_0(\Theta)$, we introduce a subset of $\mathcal{M}(\Theta)$ by 
\begin{equation}\label{def: Omage f}
    \Omega(g):=\left\{\mathrm{sign}(g(\theta))\delta_\theta:\theta\in\Theta(g)\right\}.
\end{equation}
Lemma $26$ in \cite{wang2023sparse} essentially states that if $g\in C_0(\Theta)\backslash\{0\}$, then
\begin{equation}\label{extreme points of partial infinity norm}
\mathrm{ext}\left(\partial\|\cdot\|_{\infty}(g)\right)=\Omega(g).
\end{equation}
We denote by $\widetilde{\mathbb{S}}_{\mathcal{X},\mathbf{Y}}$ the solution set of the MNI problem \eqref{MNI in measure space}. We note that the MNI problem \eqref{MNI in measure space} shares the same dual problem \eqref{dual problem} with the MNI problem \eqref{MNI in RKBS B measure M(X)}. 

\begin{proposition}\label{prop: representer theorem for MNI  in measure space}
Suppose that $m$ distinct points $x_j\in\mathbb{R}^s$, $j\in\mathbb{N}_m$, and $\mathbf{Y}\in\mathbb{R}^{t\times m}\backslash\{\mathbf{0}\}$ are given, and the functionals in $\mathbb{K}_\mathcal{X}$ are linearly independent. 
Let $\hat g$ be the function defined by \eqref{stage 1 hat g} and $\Theta(\hat g)$ defined by \eqref{def: infinity set for function new} with $g$ replaced by $\hat g$. Then for any extreme point $\hat \mu$ of the solution set $\widetilde{\mathbb{S}}_{\mathcal{X},\mathbf{Y}}$ of the MNI problem \eqref{MNI in measure space}, there exist $\gamma_\ell\in\mathbb{R}$, $\ell\in\mathbb{N}_{tm}$, with  $\sum_{\ell\in\mathbb{N}_{tm}}\gamma_\ell=\|\hat g\|_{\infty}$ and $\theta_\ell\in\Theta(\hat g)$, $\ell\in\mathbb{N}_{tm}$, such that 
    \begin{equation}\label{eq: representer for measure space}
        \hat\mu=\sum\limits_{\ell\in\mathbb{N}_{tm}}\gamma_\ell\mathrm{sign}(\hat g(\theta_\ell))\delta_{\theta_\ell}.
    \end{equation}
\end{proposition}
\begin{proof}
Note that the measure space $\mathcal{M}(\Theta)$ has the pre-dual space $C_0(\Theta)$  and the functionals $\mathcal{K}_k(x_j,\cdot)$, $k\in\mathbb{N}_t$, $j\in\mathbb{N}_m$, which belong to the pre-dual space $C_0(\Theta)$, are linearly independent. By Proposition 37 in \cite{wang2023sparse},  the function $\hat g$ defined by \eqref{stage 1 hat g} satisfies $\hat g\in\mathcal{V}_{\mathcal{N}}$ and
\begin{equation*}\label{non empty set: BN case tilde}
(\|\hat g\|_{\infty}\partial\|\cdot\|_\infty(\hat g))\cap{\widetilde{\mathcal{M}}}_{\mathcal{X},\mathbf{Y}}\neq\emptyset,
\end{equation*} 
in which the subdifferential set $\partial\|\cdot\|_\infty(\hat g)$ is a subset of the measure space $\mathcal{M}(\Theta)$. As a result, the hypothesis of Lemma \ref{lemma: representer for MNI} is satisfied. According to Lemma \ref{lemma: representer for MNI}, any extreme point $\hat \mu$ of the solution set $\widetilde{\mathbb{S}}_{\mathcal{X},\mathbf{Y}}$ of problem \eqref{MNI in measure space}, there exist $\gamma_\ell\in\mathbb{R}$, $\ell\in\mathbb{N}_{tm}$, with  $\sum_{\ell\in\mathbb{N}_{tm}}\gamma_\ell=\|\hat g\|_{\infty}$ and $u_\ell\in\mathrm{ext}(\partial \|\cdot\|_\infty(\hat g))$, $\ell\in\mathbb{N}_{tm}$, such that 
\begin{equation}\label{proof: apply general MNI theorem to C0}
\hat\mu=\sum_{\ell\in\mathbb{N}_{tm}}\gamma_\ell u_\ell.  
\end{equation}
It follows from equation \eqref{extreme points of partial infinity norm} that for each $\ell\in\mathbb{N}_{tm}$, we have that $u_\ell\in \Omega(\hat{g})$. By definition \eqref{def: Omage f} of the set $\Omega(\hat{g})$, for each $\ell\in\mathbb{N}_{tm}$, there exists $\theta_\ell\in\Theta(\hat g)$ such that $u_\ell=\mathrm{sign}(\hat g(\theta_\ell))\delta_{\theta_\ell}$. Therefore, we may rewrite the representation \eqref{proof: apply general MNI theorem to C0} of $\hat{\mu}$ as \eqref{eq: representer for measure space}. 
\end{proof}

Proposition \ref{prop: representer theorem for MNI  in measure space} provides a representation for any extreme point of the solution set of the MNI problem \eqref{MNI in measure space}. This solution can be converted via Proposition \ref{prop: MNI solution for measure space is solution for RKBS} to a solution to the MNI problem \eqref{MNI in RKBS B measure M(X)}. We present this result in the next theorem.

\begin{theorem}\label{theorem: kernel representer theorem for MNI in BN}
Suppose that $m$ distinct points $x_j\in\mathbb{R}^s$, $j\in\mathbb{N}_m$, and $\mathbf{Y}\in\mathbb{R}^{t\times m}\backslash\{\mathbf{0}\}$ are given, and the functionals in $\mathbb{K}_\mathcal{X}$ are linearly independent. Let 
$\hat g$ be the function defined by \eqref{stage 1 hat g} and $\Theta(\hat g)$ defined by \eqref{def: infinity set for function new} with $g$ replaced by $\hat g$. Then the MNI problem \eqref{MNI in RKBS B measure M(X)} has a solution $\hat f$ in the form
\begin{equation}\label{representer solution of finite linear combination of MNI}
\hat f(x)=\sum\limits_{\ell\in\mathbb{N}_{tm}}\gamma_\ell\mathrm{sign}(\hat g(\theta_\ell))\mathcal{K}(x,\theta_\ell),\ x\in\mathbb{R}^s,
\end{equation}  
for some $\gamma_\ell\in\mathbb{R}$, $\ell\in\mathbb{N}_{tm}$, with  $\sum_{\ell\in\mathbb{N}_{tm}}\gamma_\ell=\|\hat g\|_{\infty}$ and $\theta_\ell\in\Theta(\hat g)$, $\ell\in\mathbb{N}_{tm}$.
\end{theorem}


\begin{proof}
By Proposition 1 of \cite{wang2021representer}, the MNI problem \eqref{MNI in measure space} has at least one solution. That is, $\widetilde{\mathbb{S}}_{\mathcal{X},\mathbf{Y}}$ is nonempty and moreover, $\mathrm{ext}\left(\widetilde{\mathbb{S}}_{\mathcal{X},\mathbf{Y}}\right)$ is nonempty. We choose $\hat \mu\in\mathrm{ext}\left(\widetilde{\mathbb{S}}_{\mathcal{X},\mathbf{Y}}\right)$. Proposition \ref{prop: representer theorem for MNI  in measure space} ensures that there exist $\gamma_\ell\in\mathbb{R}$, $\ell\in\mathbb{N}_{tm}$, with  $\sum_{\ell\in\mathbb{N}_{tm}}\gamma_\ell=\|\hat g\|_{\infty}$ and $\theta_\ell\in\Theta(\hat g)$, $\ell\in\mathbb{N}_{tm}$, such that $\hat\mu$ may be expressed as in equation \eqref{eq: representer for measure space}.
Since $\hat{\mu}$ is a solution to problem \eqref{MNI in measure space}, we get by Proposition \ref{prop: MNI solution for measure space is solution for RKBS} that $f_{\hat\mu}=[f_{\hat\mu}^k:k\in\mathbb{N}_t]$ is a solution to the MNI problem \eqref{MNI in RKBS B measure M(X)}. By definition \eqref{Def:f_mu^k} of $f_{\hat\mu}^k$, $k\in\mathbb{N}_t$, we have that 
\begin{equation}\label{Solution_Form}
f_{\hat\mu}^k(x)=\int_\Theta \mathcal{K}_k(x,\theta)d\hat{\mu}(\theta), \quad\text{for }x\in\mathbb{R}^s,  k\in\mathbb{N}_t. 
\end{equation}
Substituting representation \eqref{eq: representer for measure space} of $\hat\mu$ into the right-hand side of equation \eqref{Solution_Form} yields that 
\begin{equation*}
f_{\hat\mu}^k(x)=\sum\limits_{\ell\in\mathbb{N}_{tm}}\gamma_\ell\mathrm{sign}(\hat g(\theta_\ell))\mathcal{K}_k(x,\theta_\ell), \quad\text{for }x\in\mathbb{R}^s,  k\in\mathbb{N}_t. 
\end{equation*}
    By letting $\hat f:=f_{\hat\mu}$, we conclude that the MNI problem \eqref{MNI in RKBS B measure M(X)} has a solution $\hat f$ in the form of \eqref{representer solution of finite linear combination of MNI}. 
\end{proof}

We now return to the regularization problem \eqref{eq: regularization problem RKBS B measure M(X)} with the goal of establishing representer theorems for its solutions. To ensure the existence of a solution, we assume that the loss function $\mathcal{Q}$ is lower
semi-continuous on $\mathbb{R}^{t\times m}$. 
Before  deriving the representer theorems, we highlight the relationship between the solutions to this problem and the MNI problem \eqref{MNI in RKBS B measure M(X)}. According to Proposition 41 of \cite{wang2021representer}, if $\hat{f}_{\mu}\in\mathcal{B}_{\mathcal{N}}$ is a solution to the regularization problem \eqref{eq: regularization problem RKBS B measure M(X)}, then it is also a
solution to the MNI problem \eqref{MNI in RKBS B measure M(X)} with $\mathbf{Y}:=\mathbf{I}_{\mathcal{X}}(\hat{f}_{\mu})$. Moreover,  any solution $\hat{f}_{\nu}\in\mathcal{B}_{\mathcal{N}}$ to the MNI problem \eqref{MNI in RKBS B measure M(X)} with $\mathbf{Y}:=\mathbf{I}_{\mathcal{X}}(\hat{f}_{\mu})$ is also  a solution to the regularization
problem \eqref{eq: regularization problem RKBS B measure M(X)}.
We denote by $\mathbb{R}_{\mathcal{X},\mathbf{Y}}$ the solution set of problem \eqref{eq: regularization problem RKBS B measure M(X)}. By Proposition \ref{Existence-of-Solution}, $\mathbb{R}_{\mathcal{X},\mathbf{Y}}$ is nonempty. We then introduce a subset $\mathcal{D}_{\mathcal{X},\mathbf{Y}}$ of $\mathbb{R}^{t\times m}$ by  
\begin{equation}\label{general Dy}
\mathcal{D}_{\mathcal{X},\mathbf{Y}}:=\mathbf{I}_{\mathcal{X}}(\mathbb{R}_{\mathcal{X},\mathbf{Y}}). 
\end{equation}
Now, recalling that $\mathbb{S}_{\mathcal{X},\mathbf{Y}}$ denotes the solution set of the MNI problem \eqref{MNI in RKBS B measure M(X)}, 
the relation between the solutions of these two problems can be represented as 
\begin{equation}\label{Proposition41-in- WangXu2}
\bigcup_{{\mathbf{Z}}\in\mathcal{D}_{\mathcal{X},\mathbf{Y}}}\mathbb{S}_{\mathcal{X},\mathbf{Z}}=\mathbb{R}_{\mathcal{X},\mathbf{Y}}. 
\end{equation}
Moreover, by Lemma 11 of \cite{wang2023sparse},  if the loss function $\mathcal{Q}$ is convex, then  
    \begin{equation}\label{relation_extreme_sets}
        \mathrm{ext}\left(\mathbb{R}_{\mathcal{X},\mathbf{Y}}\right)\subset\bigcup_{\mathbf{Z}\in\mathcal{D}_{\mathcal{X},\mathbf{Y}}}\mathrm{ext}\left(\mathbb{S}_{\mathcal{X},\mathbf{Z}}\right).
    \end{equation}

Below, we convert the representer theorem for a solution to problem \eqref{MNI in RKBS B measure M(X)} stated in Theorem \ref{theorem: direct representer theorem for MNI in BN} to that for the regularization problem \eqref{eq: regularization problem RKBS B measure M(X)} by making use of the relation between the solutions of these two problems.

\begin{theorem}\label{theorem: representer for regularization}
    Suppose that $m$ distinct points $x_j\in\mathbb{R}^s$, $j\in\mathbb{N}_m$, and $\mathbf{Y}\in\mathbb{R}^{t\times m}$ are given, $\lambda>0$. Let $\mathcal{V}_{\mathcal{N}}$ be defined by \eqref{V_span_kernel} and $\mathcal{D}_{\mathcal{X},\mathbf{Y}}$ be defined by \eqref{general Dy}.  
    \begin{enumerate}
    \item If $\mathcal{D}_{\mathcal{X},\mathbf{Y}}\neq\{\mathbf{0}\}$, then there exists a solution $\hat f$ of problem \eqref{eq: regularization problem RKBS B measure M(X)} such that 
    \begin{equation}\label{eq: regularization problem increasing case in BN}
        \hat f(x)=\sum\limits_{\ell\in\mathbb{N}_{tm}} \gamma_\ell h_\ell(x),\ x\in\mathbb{R}^s,
    \end{equation}
   for some $\hat g\in\mathcal{V}_{\mathcal{N}}$, $\gamma_\ell\in\mathbb{R}$, $\ell\in\mathbb{N}_{tm}$, with $\sum_{\ell\in\mathbb{N}_{tm}}\gamma_\ell=\|\hat g\|_{\infty}$ and $h_\ell\in\mathrm{ext}(\partial\|\cdot\|_\infty(\hat g))$, $\ell\in\mathbb{N}_{tm}$.
    \item If the loss function $\mathcal{Q}$ is convex, then every nonzero extreme point $\hat f$ of the solution set $\mathbb{R}_{\mathcal{X},\mathbf{Y}}$ of problem \eqref{eq: regularization problem RKBS B measure M(X)} has the form of
    \eqref{eq: regularization problem increasing case in BN} for some $\hat g\in\mathcal{V}_{\mathcal{N}}$, $\gamma_\ell\in\mathbb{R}$, $\ell\in\mathbb{N}_{tm}$, with  $\sum_{\ell\in\mathbb{N}_{tm}}\gamma_\ell=\|\hat g\|_{\infty}$ and $h_\ell\in\mathrm{ext}(\partial\|\cdot\|_\infty(\hat g))$, $\ell\in\mathbb{N}_{tm}$.
    \end{enumerate}
\end{theorem}

\begin{proof}
We first prove Item 1. Note that $\mathbb{R}_{\mathcal{X},\mathbf{Y}}$ is an nonempty set. It follows from the hypothesis $\mathcal{D}_{\mathcal{X},\mathbf{Y}}\neq\{\mathbf{0}\}$ that there exists $\hat{h}\in\mathbb{R}_{\mathcal{X},\mathbf{Y}}$ such that $\hat{\mathbf{Z}}:=\mathbf{I}_{\mathcal{X}}(\hat{h})\neq 0$. According to equation \eqref{Proposition41-in- WangXu2} with noting that $\hat{\mathbf{Z}}\in\mathcal{D}_{\mathcal{X},\mathbf{Y}}$, we have that $\mathbb{S}_{\mathcal{X},\hat{\mathbf{Z}}}\subset\mathbb{R}_{\mathcal{X},\mathbf{Y}}$ and thus, $\mathrm{ext}\left(\mathbb{S}_{\mathcal{X},\hat{\mathbf{Z}}}\right)\subset\mathbb{R}_{\mathcal{X},\mathbf{Y}}$. We choose $\hat f\in\mathrm{ext}\left(\mathbb{S}_{\mathcal{X},\hat{\mathbf{Z}}}\right)$ and verify that $\hat f$ can be represented as in  \eqref{eq: regularization problem increasing case in BN}. To this end, we choose $\hat{\mathbf{c}}:=[\hat{c}_{kj}:k\in\mathbb{N}_t,j\in\mathbb{N}_m]\in\mathbb{R}^{t\times m}$ to be a solution to the dual problem \eqref{dual problem} with $\mathbf{Y}$ replaced by $\hat{\mathbf{Z}}$. Let $\hat g$ be the function defined by \eqref{stage 1 hat g} with $\hat{\mathbf{c}}$. Theorem \ref{theorem: direct representer theorem for MNI in BN} ensures that $\hat f$, as an extreme point of the solution set $\mathbb{S}_{\mathcal{X},\hat{\mathbf{Z}}}$, can be represented as in  \eqref{eq: regularization problem increasing case in BN} for some $\gamma_\ell\in\mathbb{R}$, $\ell\in\mathbb{N}_{tm}$, with  $\sum_{\ell\in\mathbb{N}_{tm}}\gamma_\ell=\|\hat g\|_{\infty}$ and $h_\ell\in\mathrm{ext}(\partial\|\cdot\|_\infty(\hat g))$, $\ell\in\mathbb{N}_{tm}$.

We next show Item 2. Assume that $\hat f$ is an arbitrary nonzero extreme point of $\mathbb{R}_{\mathcal{X},\mathbf{Y}}$, that is $\hat f\in\mathrm{ext}(\mathbb{R}_{\mathcal{X},\mathbf{Y}})\backslash\{0\}$. Because the loss function $\mathcal{Q}$ is convex,  the inclusion relation  \eqref{relation_extreme_sets} is satisfied. By \eqref{relation_extreme_sets}, there exists $\hat{\mathbf{Z}}\in\mathcal{D}_{\mathcal{X},\mathbf{Y}}$ such that $\hat f\in\mathrm{ext}\left(\mathbb{S}_{\mathcal{X},\hat{\mathbf{Z}}}\right)$. Clearly, $\hat{\mathbf{Z}}\neq 0.$ Assume to the contrary that $\hat{\mathbf{Z}}= 0.$ We then must have that $\mathbb{S}_{\mathcal{X},\hat{\mathbf{Z}}}=\{0\}.$ As a result, $\hat f=0,$ which is a contradiction.
Again, let $\hat g$ be defined by \eqref{stage 1 hat g} with $\hat{\mathbf{c}}$ being a solution to problem \eqref{dual problem} with $\mathbf{Y}$ being replaced by $\hat{\mathbf{Z}}$. By Theorem \ref{theorem: direct representer theorem for MNI in BN}, we can represent $\hat f$ as in  \eqref{eq: regularization problem increasing case in BN} for some $\gamma_\ell\in\mathbb{R}$, $\ell\in\mathbb{N}_{tm}$, with  $\sum_{\ell\in\mathbb{N}_{tm}}\gamma_\ell=\|\hat g\|_{\infty}$ and $h_\ell\in\mathrm{ext}(\partial\|\cdot\|_\infty(\hat g))$, $\ell\in\mathbb{N}_{tm}$.
\end{proof}

Similarly to Theorem \ref{theorem: direct representer theorem for MNI in BN} for a solution to the MNI problem,
Theorem \ref{theorem: representer for regularization} ensures that 
if the loss function $\mathcal{Q}$ is convex, each extreme point of the solution set of the regularization problem \eqref{eq: regularization problem RKBS B measure M(X)} lays in a {\it finite} dimensional manifold spanned by $tm$ elements  $h_\ell\in\mathrm{ext}(\partial\|\cdot\|_\infty(\hat g))$ and it has an explicit, data-dependent representation.

We further show that there exists a solution to the regularization problem \eqref{eq: regularization problem RKBS B measure M(X)} that can be represented as kernel expansions in terms of the network parameter features determined by the
training data.

\begin{theorem}\label{theorem: kernel representation for regularization problems}
Suppose that $m$ distinct points $x_j\in\mathbb{R}^s$, $j\in\mathbb{N}_m$, and $\mathbf{Y}\in\mathbb{R}^{t\times m}$ are given, $\lambda>0$. Let $\mathcal{V}_{\mathcal{N}}$ be defined by \eqref{V_span_kernel}, $\mathcal{D}_{\mathcal{X},\mathbf{Y}}$ be defined by \eqref{general Dy} and $\mathcal{D}_{\mathcal{X},\mathbf{Y}}\neq\{\mathbf{0}\}$.  Then there exists a solution $\hat f$ of problem \eqref{eq: regularization problem RKBS B measure M(X)} such that 
\begin{equation}\label{eq: kernel representation regularization problem increasing case in BN}
\hat f(x)=\sum\limits_{\ell\in\mathbb{N}_{tm}} \gamma_\ell\mathrm{sign}(\hat g(\theta_\ell)) \mathcal{K}(x,\theta_\ell),\ x\in\mathbb{R}^s,
\end{equation}
for some $\hat g\in\mathcal{V}_{\mathcal{N}}$,  $\gamma_\ell\in\mathbb{R}$, $\ell\in\mathbb{N}_{tm}$, with  $\sum_{\ell\in\mathbb{N}_{tm}}\gamma_\ell=\|\hat g\|_{\infty}$ and $\theta_\ell\in\Theta(\hat g)$, $\ell\in\mathbb{N}_{tm}$.
\end{theorem}

\begin{proof}
Since $\mathbb{R}_{\mathcal{X},\mathbf{Y}}$ is nonempty and $\mathcal{D}_{\mathcal{X},\mathbf{Y}}\neq\{\mathbf{0}\}$, there exists $\hat h \in \mathbb{R}_{\mathcal{X},\mathbf{Y}}$ such that $\hat{\mathbf{Z}}:=\mathbf{I}_{\mathcal{X}}(\hat h)\neq \mathbf{0}$. We choose $\hat g$ in the form of \eqref{stage 1 hat g}, where $\hat{\mathbf{c}}$ is a solution to problem \eqref{dual problem} with $\mathbf{Y}$ being replaced by $\hat{\mathbf{Z}}$. According to Theorem \ref{theorem: kernel representer theorem for MNI in BN}, the MNI problem \eqref{MNI in RKBS B measure M(X)} with $\mathbf{Y}:=\hat{\mathbf{Z}}$ has a solution  $\hat f$ in the form of \eqref{eq: kernel representation regularization problem increasing case in BN}, for some $\gamma_\ell\in\mathbb{R}$, $\ell\in\mathbb{N}_{tm}$, with  $\sum_{\ell\in\mathbb{N}_{tm}}\gamma_\ell=\|\hat g\|_{\infty}$ and $\theta_\ell\in\Theta(\hat g)$, $\ell\in\mathbb{N}_{tm}$. In other words, $\hat f\in \mathbb{S}_{\mathcal{X},\hat{\mathbf{Z}}}$ and it has the form of \eqref{eq: kernel representation regularization problem increasing case in BN}. It follows from relation \eqref{Proposition41-in- WangXu2} that $\mathbb{S}_{\mathcal{X},\hat{\mathbf{Z}}}\subset\mathbb{R}_{\mathcal{X},\mathbf{Y}}$, which implies that $\hat f\in\mathbb{R}_{\mathcal{X},\mathbf{Y}}$.
\end{proof}

The representer theorems—Theorem \ref{theorem: kernel representer theorem for MNI in BN} for the MNI problem \eqref{MNI in RKBS B measure M(X)} and Theorem \ref{theorem: kernel representation for regularization problems} for the regularization problem \eqref{eq: regularization problem RKBS B measure M(X)}—establish that solutions to these  problems can be expressed as a finite sum of kernel expansions in terms of the network parameter features determined by the  training data. In other words, these representer theorems are explicit and data-dependent. 
A framework for constructing integral scalar-valued RKBS and proving a representer theorem (Theorem 3.9) for the regularization problem in the RKBS was proposed in \cite{bartolucci2023understanding}. This result was derived using a data-independent representer theorem for a variational problem in a Banach space established in \cite{boyer2019representer,bredies2020sparsity}, making Theorem 3.9 of \cite{bartolucci2023understanding} inherently data-independent. 
While one could apply a generalized version of this framework to the regularization problem \eqref{eq: regularization problem RKBS B measure M(X)} to obtain a data-independent representer theorem, data-dependent representer theorems are highly desirable. Theorems \ref{theorem: kernel representer theorem for MNI in BN} and  \ref{theorem: kernel representation for regularization problems} established using the framework proposed in \cite{wang2021representer, wang2023sparse}, achieve this by explicitly incorporating data features.
More precisely,  
the representer theorems presented here strengthen the results that could be obtained from \cite{bartolucci2023understanding} by specifying that the network parameters $\theta_k$ belong to the data-dependent set $\Theta(\hat{g})$, where $\hat{g}$ is a kernel expansion in terms of given data. 
Beyond offering insight into the structure of learning solutions, these results pave the way for new approaches to constructing such solutions, which will be discussed next.

To conclude this section, we elaborate on ideas of leveraging the representer theorems for deep learning. Recall that in a reproducing kernel Hilbert space, representer theorems \cite{scholkopf2001generalized} significantly simplify learning problems by reducing an infinite-dimensional optimization problem to a finite-dimensional one. More importantly, they reveal that the solutions to these problems lie within the space spanned by the sampled features--specifically, the reproducing
kernel evaluated at the given data points-- thereby enabling the development of practical and implementable learning algorithms. 

This fundamental result has been extended to learning problems in reproducing kernel Banach spaces \cite{bartolucci2023understanding,parhi2021banach,xu2019generalized,zhang2009reproducing,wang2021representer,wang2023sparse}. While the representer theorem in Banach spaces also reduces an infinite-dimensional optimization problem to a finite-dimensional one, it differs from its counterpart in Hilbert spaces. Specifically, by the duality mapping \cite{Cioranescu1990geometry}, every element $f$ in a Banach space corresponds to a subset of its dual space, whose elements can be referred to as {\it dual functionals} associated with $f$. The essence of the representer theorem in Banach spaces is that it represents a dual functional of the learning solution as a finite linear combination of the sampled features. In the special case of Hilbert spaces, where the dual space coincides with the space itself, the dual functionals of the learning solution are identical to the solution itself, naturally leading to a direct representation of the solution as a finite linear combination of the sampled features. However, in a general Banach space, the learning solution and its dual functionals are not identical, in fact, they belong to different spaces, making it impossible to express the solution directly in this form. Several research efforts have demonstrated how to design implementable algorithms for solving specific learning problems in particular Banach spaces using the representer theorem. 
For example, the MNI problem in $\ell_1(\mathbb{N})$ was solved in \cite{cheng2021minimum} by first determining  a dual functional in the pre-dual space via linear programming,  then recovering the MNI solution from the dual functional by solving a related linear system. This duality-based approach was extended in \cite{ChengWangXu2023} to a broader class of regularization problems by reformulating them as MNI problems on related quotient spaces and solving them using the methodology developed in  \cite{cheng2021minimum}. An alternative approach was introduced in \cite{wang2021representer}, where the solutions to these problems were characterized as fixed-points of finite-dimensional nonlinear mappings, known as representer theorems in fixed-point equation form, enabling the development of iterative schemes to solve them.


The hypothesis spaces introduced with RKBSs in the previous section, along with the kernel-based, data-dependent representer theorems established earlier in this section for deep learning, enable the development of practical algorithms for deep learning problems by leveraging the methodology for learning in RKBSs. We illustrate this approach using the MNI problem as an example.

First, we determine the dual functional $\hat g$, defined as in \eqref{stage 1 hat g}, by solving the dual problem \eqref{dual problem} associated with the MNI problem. Notably, $\hat g$ takes the form of a finite linear combination of the sampled features. Second, we obtain the solution $\hat f$ in the form \eqref{representer solution of finite linear combination of MNI} by determining the network parameter features $\theta_l$, $\ell\in\mathbb{N}_{tm}$, and the coefficients $\beta_\ell:=\gamma_\ell{\rm sign}(\hat{g}(\theta_\ell))$, $\ell\in\mathbb{N}_{tm}$. According to representer theorem \ref{theorem: kernel representer theorem for MNI in BN}, the network parameter features $\theta_l$, $\ell\in\mathbb{N}_{tm}$, can be selected in the set  $\Theta(\hat g)$. Since the solution satisfies the the interpolation condition $f_{\mu}(x_j)=y_j$, $j\in\mathbb{N}_m$, and the sequence $\gamma_\ell\in\mathbb{R}$, $\ell\in\mathbb{N}_{tm}$ satisfies $\sum_{\ell\in\mathbb{N}_{tm}}\gamma_\ell=\|\hat g\|_{\infty}$, these quantities can be determined by solving a linear programming problem. Developing implementable algorithms for solving deep learning problems in the  hypothesis space $\mathcal{B}_{\mathcal{N}}$ based on the representer theorems remains an open area of research, which we leave for future work.

Furthermore, for both MNI and regularization problems, our data-dependent representer theorem ensures the existence of a solution in the form of kernel expansions with data-dependent coefficients and parameters. As an alternative approach, one can directly treat these coefficients and parameters as trainable variables and numerically solve the following optimization problem
\begin{equation}\label{multi linear DNN model}
    \min\left\{\mathcal{L}\left(\sum_{\ell\in\mathbb{N}_{tm}}\beta_\ell\mathcal{N}(\cdot,\theta_\ell),\mathbb{D}_m\right):\beta_\ell\in\mathbb{R},\theta_\ell\in\Theta_{\mathbb{W}},\ell\in\mathbb{N}_{tm}\right\}. 
\end{equation}
Standard deep learning training techniques, such as stochastic gradient descent, are well-suited for numerically solving \eqref{multi linear DNN model}. The development of feasible algorithms for implementing these ideas is left for future research.



\section{Concluding Remarks}\label{section: concluding remarks}
In this paper, we introduced the hypothesis
space $\mathcal{B}_{\mathcal{N}}$ for deep learning. This space is a vector-valued RKBS with a unique reproducing kernel $\mathcal{K}$, formed as the weak* completion of the vector space $\mathcal{B}_{\mathbb{W}}$, which is the linear span of the original set $\mathcal{A}_{\mathbb{W}}$ in deep learning. The introduction of this hypothesis space provides a mathematical framework that offers deeper insights into deep learning. Specifically, by leveraging $\mathcal{B}_{\mathcal{N}}$, we have developed representer theorems for solutions to two deep learning models: the MNI and regularization problems involving deep neural networks.

In the remainder of this section, we discuss the relations among different learning models studied in this paper. 
%
Suppose that the loss function $\mathcal{L}(f_\mu,\mathbb{D}_m)$ takes the form of \eqref{loss:An example}, where  $\mathcal{Q}:\mathbb{R}^{t\times m}\to \mathbb{R}_+$ satisfies $\mathcal{Q}(0)=0$. If $\hat{f}_\mu\in \mathcal{B}_\mathcal{N}$ is a solution to the MNI problem \eqref{MNI-original}, then $\mathcal{L}(\hat{f}_\mu, \mathbb{D}_m)=0$, implying that $\hat{f}_\mu$ also solves the learning model \eqref{LearningMethodinRKBS}. Furthermore, let $f_{\mu,\lambda}$ be a solution to the regularized learning problem 
\eqref{eq: regularization problem RKBS B measure M(X)}, and  define $f_{\mu,0}$ as the limit of $f_{\mu,\lambda}$ as $\lambda\to 0$ provided the limit exists. It follows that $f_{\mu,0}$ is a solution to the learning model \eqref{LearningMethodinRKBS}. Thus, the models \eqref{eq: regularization problem RKBS B measure M(X)} and 
\eqref{MNI-original} serve as stable alternatives to \eqref{LearningMethodinRKBS}, which may suffer from instability.

Learning within the RKBS space $\mathcal{B}_{\mathcal{N}}$ offers several advantages.
\begin{description}
\item[1. Well-Defined Learning Framework:] 
The standard learning model \eqref{BasicLearningMethod-equivalent}, widely used in machine learning, may lack solutions due to the absence of algebraic or topological structure in the original set $\mathcal{A}_\mathbb{W}$. However,   $\mathcal{B}_\mathcal{N}$ is the weak* completion of the linear span $\mathcal{B}_\mathbb{W}$ and possesses well-defined algebraic and topological properties, making it a natural hypothesis space for deep learning problems.
\item[2. Guaranteed Existence of Solutions:] Unlike the original learning model \eqref{BasicLearningMethod-equivalent}, the regularized learning model \eqref{eq: regularization problem RKBS B measure M(X)} is guaranteed to have a solution under mild conditions, thanks to the completeness and structure of $\mathcal{B}_\mathcal{N}$.
\item[3. Kernel-Based and Data-Dependent Representer Theorems:] 
The reproducing kernel of the RKBS $\mathcal{B}_{\mathcal{N}}$ enables solutions in this RKBS to be expressed in terms of kernel expansions, leading to representer theorems for deep learning. These theorems reveal that while the learning models in $\mathcal{B}_{\mathcal{N}}$ are formulated in an infinite-dimensional space, their solutions reside in finite-dimensional manifolds and can be expressed as finite sums of kernel expansions based on training data. 
\end{description}
In conclusion, introducing the hypothesis space $\mathcal{B}_\mathcal{N}$ to deep learning not only enhances our understanding of its mathematical foundations but also provides a robust framework for analyzing and solving deep learning problems.

To conclude this section, we summarize the key contribution of this paper as follows:
\begin{description}
   \item[1. Introduction of a Hypothesis Space for Deep Learning.] We introduce a novel hypothesis space, denoted as $\mathcal{B}_{\mathcal{N}}$, for deep learning, with a fixed depth and layer widths. This space is constructed by taking the linear span of the set of the corresponding neural networks and completing it in a weak* topology. The introduction of the hypothesis space enables a more rigorous mathematical analysis of deep learning and facilitates the development of more interpretable training algorithms.
   Our construction extends existing results, such as the linear space for shallow learning introducing in \cite{bartolucci2023understanding}, while differing from the {\it nonlinear spaces} proposed in \cite{shenouda2024variation,bartolucci2024neural} for deep learning, which are formed through compositions of linear spaces for shallow networks. 
\item[2. Identification of the Hypothesis Space as an  RKBS.] We identify  $\mathcal{B}_{\mathcal{N}}$ as a vector-valued RKBS with an asymmetric kernel $\mathcal{K}(x, \theta)$, which is given by the product of the network $\mathcal{N}(x,\theta)$ with an appropriate weight function $\rho(\theta)$. The identification enables the representation of learning solutions in terms of the kernel, providing a functional-analytic framework for understanding deep learning.
\item[3. Representer Theorems of Learning Solutions.]  We establish  data-dependent, kernel-based representer theorems  for two types of learning models:  minimum norm interpolation and regularized learning within the introduced  hypothesis space $\mathcal{B}_{\mathcal{N}}$. These theorems demonstrate that the solutions to these models can be expressed as a finite sum of the kernel expansions based on training data. This result contrasts with existing work \cite{bartolucci2023understanding, parhi2021banach, shenouda2024variation,bartolucci2024neural,parhi2022what,unser2019representer}, which is either restricted to shallow learning or data-independent. Our theorems yield more structured and interpretable learning solutions, offering deeper insights into how data influences the structure of learning solutions within the hypothesis space.

\end{description}

\appendix
\section{Necessary Notions of Convex Analysis}
In this appendix, we review key concepts from convex analysis that are essential for establishing the representer theorems in Section \ref{section: representer theorems}.

\noindent{\bf Convex Sets and Extreme Points.} Let $\mathbb{X}$ be a Hausdorff locally convex topological vector space. A subset $A$ of $\mathbb{X}$ is called convex if for any $x,y\in A$ and any $t\in[0,1]$, the point $t x+(1-t)y$ also belongs to $A$. The convex hull of $A$ of  $\mathbb{X}$, denoted by $\mathrm{co}(A)$, is the smallest convex set containing $A$. The closed convex hull of $A$, denoted by $\overline{\mathrm{co}}(A)$, is the smallest closed convex set containing $A$, where the closure is taken with respect to the topology of $\mathbb{X}$.

Now, suppose that $A$ is a nonempty closed convex subset of $\mathbb{X}$. An element $z\in A$ is called an extreme point of $A$ if whenever $x,y\in A$ satisfy $tx+(1-t)y=z$ for some $t\in(0,1)$, it follows that $x=y=z$. We denote the set of extreme points of $A$ by $\mathrm{ext}(A)$. A fundamental result in convex analysis, the Krein-Milman theorem \cite{megginson2012introduction}, states that if 
$A$ is a nonempty compact convex subset of $\mathbb{X}$, then $A$ is the closed convex hull of its set of extreme points, that is, 
$$
A=\overline{\mathrm{co}}\left(\mathrm{ext}(A)\right).
$$


\noindent{\bf Subdifferentials.}
Let $B$ be a Banach space with norm $\|\cdot\|_B$. The norm function $\|\cdot\|_B$ is a convex function on $B$. The subdifferential of $\|\cdot\|_B$ at any nonzero element $f\in B\backslash\{0\}$ is defined as  \begin{equation*}\label{subdifferential = norming functional}
  \partial \|\cdot\|_B(f):=\left\{\nu\in B^*:\|\nu\|_{B^*}=1,\langle \nu,f\rangle_{B}=\|f\|_{B}\right\}.
\end{equation*}
Here, $B^*$ denotes the dual space  of $B$, which when equipped with the weak${}^*$ topology, forms a Hausdorff locally convex topological vector space. For any $f\in B\backslash\{0\}$, the subdifferential set ${\partial\|\cdot\|_B(f)}$ is a convex and weakly${}^*$ compact in $B^*$. By the Krein-Milman theorem, it follows that 
\begin{equation*}\label{representation subdifferential set}
    \partial \|\cdot\|_B(f)=\overline{\mathrm{co}}(\mathrm{ext}(\partial \|\cdot\|_{B}(f))),
\end{equation*}
where the closure is taken with respect to the weak$^*$ topology of $B^*$.

\section{A Representer Theorem for the MNI Problem in a Banach Space}
In this appendix, we recall the explicit, data-dependent representer theorem for the MNI problem in a general Banach space setting, as established in \cite{wang2023sparse}. This result provides a foundational tool for deriving the representer theorems presented in Section \ref{section: representer theorems}.

\noindent{\bf Formulation of the MNI Problem.}
We consider the MNI problem in a general Banach space that has a pre-dual space. Let $B$ be a Banach space with a pre-dual space $B_*$. Suppose that $\nu_j$, $j\in\mathbb{N}_n$, are linearly independent elements in $B_*$, and let $\mathbf{z}:=[z_j:j\in\mathbb{N}_n]\in\mathbb{R}^n$ be a given vector. Define 
$$
\mathcal{V}:=\mathrm{span}\{\nu_j:j\in\mathbb{N}_n\}.
$$
Next, we introduce the linear operator $\mathcal{L}: B \rightarrow \mathbb{R}^{n}$ given by  
\begin{equation*}\label{general operator L in lemma}
    \mathcal{L}(f):=\left[\left\langle\nu_{j}, f\right\rangle_{B}: j \in \mathbb{N}_{n}\right], \text { for all } f \in B.
\end{equation*}
The feasible set of the MNI problem is
$$
{M_{\mathbf{z}}}:=\{f \in B: \mathcal{L}(f)=\mathbf{z}\}.
$$
The MNI problem with the given data $\{(\nu_j,y_j):j\in\mathbb{N}_n\}$, as studied in \cite{wang2023sparse}, is formulated as   
\begin{equation}\label{general MNI in lemma}
    \inf \left\{\|f\|_{B}: f \in {M_{\mathbf{z}}}\right\}.
\end{equation}

 \noindent{\bf A Representer Theorem for the MNI Problem.}
A key result from Proposition 7 in \cite{wang2023sparse} characterizes the extreme points of the solution set of problem \eqref{general MNI in lemma}. This result is formally stated in the following lemma.

\begin{lemma}\label{lemma: representer for MNI}
Let $B$ be a Banach space with a pre-dual $B_{*}$. Suppose that $\nu_{j} \in B_{*}$, $j \in \mathbb{N}_{n}$, are linearly independent, and let $\mathbf{z} \in \mathbb{R}^{n}\backslash\{\mathbf{0}\}$. If  $\mathcal{V}$ and   ${M_{\mathbf{z}}}$ are defined as above, and there exists $\hat\nu\in\mathcal{V}$ satisfying 
\begin{equation}\label{Non-empty-set in lemma}
(\|\hat\nu\|_{B_*}\partial\|\cdot\|_{B_*}(\hat\nu))\cap{M_{\mathbf{z}}}\neq\emptyset,
\end{equation}
then for any extreme point $\hat f$ of the solution set of \eqref{general MNI in lemma}, there exist coefficients $\gamma_j\in\mathbb{R}$, $j\in\mathbb{N}_{n}$, satisfying  $$
\sum_{j\in\mathbb{N}_{n}}\gamma_j=\|\hat\nu\|_{B_*}
$$ 
and elements $u_j\in\mathrm{ext}\left(\partial\|\cdot\|_{B_*}(\hat\nu)\right)$, $j\in\mathbb{N}_{n}$, such that
\begin{equation*}\label{eq: expansion of p in lemma}
\hat f=\sum\limits_{j\in\mathbb{N}_{n}} \gamma_j u_j.
\end{equation*}
\end{lemma}

\noindent{\bf Remarks.}
As noted in \cite{wang2023sparse}, the element $\hat\nu$ satisfying \eqref{Non-empty-set in lemma} can be obtained by solving a dual problem associated with \eqref{general MNI in lemma}. Furthermore, the solution set is a nonempty, convex and weakly$^*$ compact subset of $B$. By the Krein-Milman theorem, this implies that the set of extreme points of the solution set is nonempty. Additionally, any solution to problem \eqref{general MNI in lemma} can be expressed as the weak$^*$ limit of a sequence in the convex hull of the set of extreme points. 

\section*{Acknowledgement}

The authors would like to express their gratitude to Professor Raymond Cheng for helpful discussion on the existence of a Banach space of functions that is isometrically isomorphic to a given Banach space. R. Wang is supported in part by the Natural Science Foundation of China under grant 12171202 and by the National Key Research and Development Program of China under grants 2020YFA0714101; Y. Xu is supported in part by the US National Science Foundation under grant DMS-2208386, and by the US National Institutes of Health under grant R21CA263876. M. Yan is supported in part by Cheng Fund from College of Science at Old Dominion University. All correspondence should be addressed to Y. Xu.









\end{document}